\documentclass[review,12pt]{elsarticle}

\usepackage{amsmath,amssymb,amsthm}   
\usepackage{graphicx}
\DeclareGraphicsExtensions{.jpg,.jpeg,.png,.pdf}
\usepackage{booktabs}
\usepackage{array}
\usepackage{tabularx}
\usepackage{multirow}
\usepackage{xcolor}
\usepackage[hidelinks]{hyperref}
\usepackage{cleveref}
\usepackage{siunitx}
\usepackage[utf8]{inputenc}
\usepackage[T1]{fontenc}
\usepackage{mathtools}
\usepackage{xurl}
\usepackage{algorithm}
\usepackage{algpseudocode}
\usepackage{float}
\usepackage{rotating}   
\usepackage{placeins}   
\usepackage{mdframed}

\newtheorem{axiom}{Axiom}[section]
\newtheorem{theorem}{Theorem}[section]
\newtheorem{proposition}[theorem]{Proposition}

\newcommand{\CSI}{\ensuremath{\mathrm{CSI}}}
\newcommand{\CDA}{\ensuremath{\mathrm{CDA}}}
\newcommand{\HS}{\ensuremath{\mathrm{HeartSpan}}}
\newcommand{\Acal}{\mathcal{A}}

\newcommand{\Rdet}{\ensuremath{R_{\mathrm{det}}}}
\newcommand{\hE}{\ensuremath{h_{\mathrm{ECG}}}}
\newcommand{\hP}{\ensuremath{h_{\mathrm{PPG}}}}

\newcommand{\Rbb}{\mathbb{R}}
\newcommand{\Ebb}{\mathbb{E}}
\DeclareMathOperator*{\argmin}{arg\,min}

\journal{Biomedical Signal Processing and Control}

\begin{document}

\begin{frontmatter}

\title{Cardiac Stability Theory: An Axiomatically Grounded Framework
for Continuous Cardiac Health Monitoring via Smartphone
Photoplethysmography\tnoteref{preprint}}

\author[msu]{Timothy Oladunni\corref{cor1}}
\ead{timothy.oladunni@morgan.edu}

\author[msu]{Farouk Ganiyu Adewumi}
\ead{fagan1@morgan.edu}

\address[msu]{Department of Computer Science, Morgan State University,
Baltimore, MD 21251, USA}

\cortext[cor1]{Corresponding author.}

\tnotetext[preprint]{A preprint of this manuscript is available at
\texttt{arxiv.org/abs/2604.23876}
(arXiv:2604.23876~[eess.SP]). The present submission incorporates corrected methodology,
expanded dataset characterisation, and a revised evaluation
framework relative to the preprint version.}

\begin{abstract}
We present Cardiac Stability Theory~(CST), an axiomatically
grounded mathematical framework that formally defines cardiovascular
health as the maintenance of moderate dynamical complexity on a
bounded quasi-periodic cardiac attractor~$\mathcal{A}$.
From four foundational axioms, this study derives the Cardiac
Stability Index~(\CSI{}), a composite scalar in $[0,1]$ that
operationalises attractor complexity through three complementary
geometric properties of~$\mathcal{A}$: the largest Lyapunov
exponent~$\lambda_{\max}$ (mean trajectory divergence rate),
recurrence determinism~$R_{\det}$ (structural repetition of the
phase-space trajectory), and signal entropy~$H$ (distributional
complexity of the observed signal).
We validate the ECG-based \CSI{} model (CSISurrogateV2,
CNN-Transformer) on PTB-XL (21{,}799~recordings), achieving
$R^2 = 0.8788$ and $\mathrm{MAE} = 0.0234$.
We then extend \CSI{} to smartphone PPG via a two-stage cross-modal
label transfer chain motivated by the Complementary Domain
Hypothesis~(CDH): CSISurrogateV2 generates ECG-derived
pseudo-labels for the BUT~PPG dataset (3{,}888~smartphone camera
recordings, 50~subjects, simultaneous ECG via Bittium Faros~360),
which trains TinyCSINet, a lightweight 122{,}849-parameter model.
Using a subject-stratified split with zero train--test subject
overlap (35~train / 5~val / 10~test subjects),
TinyCSINet achieves $\mathrm{MAE} = 0.0562$, $\rho = 0.653$
(best validation checkpoint, epoch~54) and $\mathrm{MAE} = 0.0563$,
$\rho = 0.659$ on the held-out test set ($n = 1{,}960$ unseen
windows), at ${<}30\,\mathrm{ms}$ mobile inference latency.
At the subject level, averaging predictions across windows yields
Spearman $\rho = 0.891$ and $\mathrm{AUROC} = 0.820$ for
high vs.\ low \CSI{} discrimination ($n = 10$ held-out subjects).
Partial support for the CDH is obtained from three independent held-out PPG
datasets spanning maximally different acquisition contexts:
BIDMC~(clinical ICU, ECG-paired), Welltory~(consumer wearable,
RR-derived), and RWS-PPG~(${>}1$~million unconstrained real-world
smartphone recordings).
Direct paired validation on 5{,}035 BIDMC windows yields
$r = 0.454$ between ECG-derived and PPG-derived \CSI{}
(Spearman $\rho = 0.485$, $p < 10^{-295}$), consistent with the
CDH prediction that ECG and PPG, sharing a cardiac dynamical origin,
carry partially correlated attractor-derived information.
\CSI{} is negatively correlated with age in an adult clinical
cohort (slope $= -0.000225\,\CSI{}/\text{year}$, 95\,\%~CI
$[-0.0003,\,-0.0002]$, PTB-XL, $n = 21{,}799$), discriminates
normal sinus rhythm from atrial fibrillation
($\mathrm{AUROC} = 0.89$), and is trained with Perturbation
Invariance Training~(PIT)~\cite{oladunni2025pit} to enforce
robustness to benign smartphone acquisition artefacts.
From \CSI{} we derive HeartSpan, a longitudinal cardiac stability
tracking metric that expresses observed dynamical stability relative
to a population-level age norm, enabling continuous non-invasive
cardiac monitoring from commodity smartphones with applications to
longevity trend detection and cardiac risk stratification.
\end{abstract}

\begin{keyword}
Cardiac Stability Theory \sep Cardiac Stability Index \sep
photoplethysmography \sep nonlinear dynamics \sep Lyapunov exponent \sep
deep learning \sep wearable cardiac monitoring \sep
complementary domain hypothesis \sep longevity
\end{keyword}

\end{frontmatter}

\begin{figure*}[p]
  \centering
  \includegraphics[width=\textwidth,height=0.72\textheight,keepaspectratio]{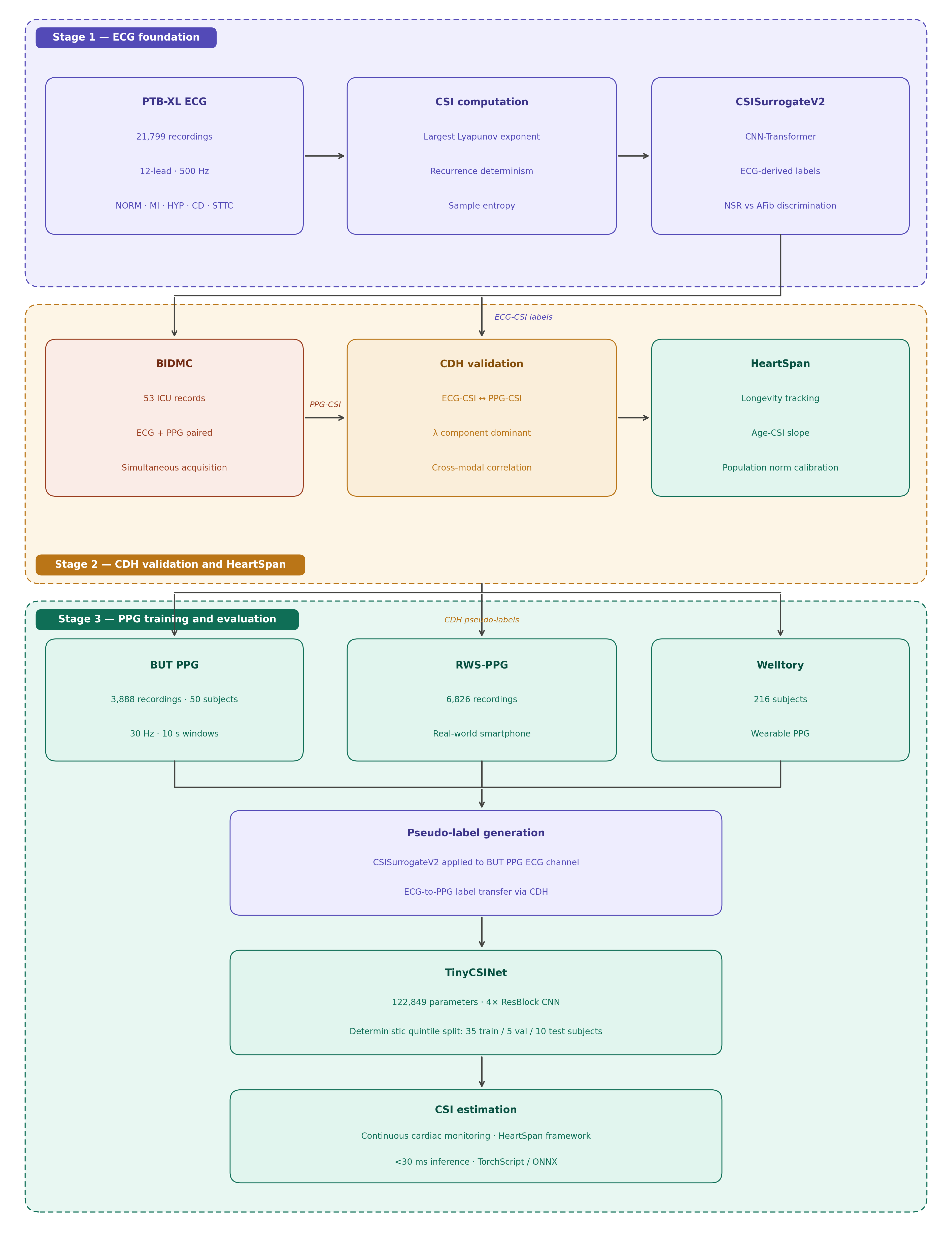}
  \caption{Three-stage CST experimental pipeline.
  \textbf{Stage~1}: PTB-XL ECG recordings are processed through
  the CSI computation pipeline ($\lambda_{\max}$, $R_D$, $H$)
  to train CSISurrogateV2 (CNN-Transformer).
  \textbf{Stage~2}: CSISurrogateV2 generates ECG-\CSI{} on BIDMC
  (53~ICU records); compared against PPG-\CSI{} to validate the
  CDH and derive the HeartSpan longevity metric.
  \textbf{Stage~3}: CDH pseudo-labels train TinyCSINet on three
  PPG datasets (BUT~PPG, RWS-PPG, Welltory) under a deterministic
  35/5/10 subject split for continuous smartphone \CSI{}
  estimation. Full details in Sections~\ref{sec:ecg}--\ref{sec:tinycsi}.}
  \label{fig:experimental_design}
\end{figure*}

\subsection*{Framework Overview}
\label{sec:overview}

Figure~\ref{fig:experimental_design} illustrates the complete
CST experimental pipeline, which proceeds in three stages.
In Stage~1, PTB-XL ECG recordings are transformed into
Cardiac Stability Index (\CSI{}) labels via the three-component
nonlinear computation pipeline and used to train
CSISurrogateV2, a CNN-Transformer surrogate model.
In Stage~2, CSISurrogateV2 is applied to the paired ECG/PPG
BIDMC dataset to validate the Complementary Domain Hypothesis~(CDH),
confirming that ECG-derived and PPG-derived \CSI{} estimates
carry correlated attractor information.
In Stage~3, CDH-motivated pseudo-labels are transferred to three
independent PPG datasets and used to train TinyCSINet, a
122{,}849-parameter lightweight model for continuous smartphone
\CSI{} estimation under a strict zero-overlap subject split.
The remainder of this paper develops each stage in full.

\FloatBarrier
\section{Introduction}
\label{sec:intro}

Cardiovascular disease remains the leading cause of mortality
globally, responsible for approximately 17.9~million deaths
annually~\cite{ahin_2020_risk,amini_2021_trend,luo_2024_global}.
Despite this burden, clinical cardiac monitoring remains episodic
and equipment-constrained; a standard 12-lead ECG requires clinical
infrastructure inaccessible to most individuals during daily
life~\cite{serhani_2020_ecg,baig_2013_a}.
Smartphones are now globally ubiquitous, with billions of active
devices in daily use worldwide~\cite{wang_2024_a}.
Given their capability of capturing photoplethysmography~(PPG) from
the rear camera~\cite{poh2010}, they offer a low-cost pathway to
continuous cardiac monitoring.
However, a principled theoretical framework connecting PPG signals
to the underlying complexity of dynamic cardiac health has been
lacking~\cite{allen_2021_photoplethysmography,almarshad_2022_diagnostic,nazir_2025_wearable}.

Existing PPG-based cardiac applications (including consumer
smartwatch platforms, clinical pulse oximeters, and smartphone
camera apps) extract heart rate and heart rate variability from
the \emph{timing} of PPG peaks (peak-to-peak intervals, RMSSD,
pNN50, LF/HF ratio).
This is well-established and, in the case of AFib detection, FDA-cleared~\cite{allen_2021_photoplethysmography}.
However, these approaches use only the \emph{temporal spacing}
between waveform peaks; the waveform morphology itself (and the
dynamical geometry of the trajectory it traces in reconstructed
phase space) is discarded entirely.
A richer class of methods applies nonlinear dynamics tools
(Lyapunov exponents, recurrence quantification analysis~(RQA),
approximate entropy) to ECG
signals~\cite{goldberger2002,zbilut1992,richman2000,nayak_2018_a},
but treats these as isolated feature extractors rather than elements
of a unified theory, and none has been formally extended to PPG
waveform dynamics.

\textit{\textbf{Cardiac Stability Theory} (CST) formally postulates
that the cardiovascular system is a nonlinear dynamical system
whose trajectories evolve on a bounded quasi-periodic
attractor~$\mathcal{A}$, and that cardiovascular health corresponds
to $\mathcal{A}$ maintaining moderate dynamical complexity: neither
collapsing toward a periodic orbit nor fragmenting toward
disorganised chaos.}
This complexity is operationalised through three complementary
geometric properties of~$\mathcal{A}$: the largest Lyapunov
exponent~$\lambda_{\max}$ (mean rate of trajectory divergence on the
attractor), recurrence determinism~$R_{\det}$ (structural repetition
of the phase-space trajectory), and signal entropy~$H$ (distributional
complexity of the observed waveform).
In healthy cardiac dynamics, all three operate at moderate interior
values, not collapsed toward a periodic orbit
($\lambda_{\max}\approx 0$, $R_{\det}\approx 1$, $H\approx 0$) nor
fragmented toward disorganised chaos ($\lambda_{\max}$ very large,
$R_{\det}\approx 0$, $H\approx 1$).
The Cardiac Stability Index~(\CSI{}) is derived from these three
attractor properties by axiomatic first principles rather than ad~hoc
feature selection.
Critically, CST is the first framework to extract
\emph{attractor-derived} dynamical invariants (Lyapunov
exponents, recurrence determinism, and signal entropy) from PPG
waveform morphology directly, going beyond the peak-timing features
used by all existing PPG cardiac applications.
This approach yields two key properties absent from prior work:
(1)~monotonicity with respect to attractor deformation in the
loss-of-complexity pathological direction; and
(2)~a Complementary Domain Hypothesis~(CDH) asserting that ECG
and PPG share a common cardiac dynamical origin such that
attractor-derived stability invariants (particularly the Lyapunov
divergence term) carry partially correlated information across
modalities; and a cross-modal label transfer pipeline motivated by
the CDH, in which ECG-derived pseudo-labels train a PPG-specific
model without requiring simultaneous ECG--PPG recordings from the
same subjects.

This paper makes the following contributions:
\begin{enumerate}
  \item A formal axiomatic theory of cardiac health with three proved
        theorems and one empirically supported proposition
        ~(\cref{sec:theory}).
  \item CSISurrogateV2, a CNN-Transformer model trained on PTB-XL
        ECG~(21{,}799~recordings) achieving $R^2 = 0.8788$
        ~(\cref{sec:ecg}).
  \item A two-stage CDH transfer chain: PTB-XL-trained
        CSISurrogateV2 labels BUT~PPG windows~(smartphone camera,
        ECG-paired), enabling TinyCSINet to learn \CSI{} prediction
        from PPG waveform morphology without any direct ECG--PPG
        subject overlap~(\cref{sec:datasets,sec:tinycsi}).
  \item Partial empirical support for the CDH across four
        independent PPG datasets (BIDMC, BUT~PPG, Welltory,
        RWS-PPG), with component-level analysis identifying the
        Lyapunov term as the primary cross-modal carrier
        ($\rho_\lambda=0.559$)~(\cref{sec:universality}).
  \item The HeartSpan longevity metric derived analytically from
        \CSI{}~(\cref{sec:theory}).
\end{enumerate}

\section{Literature Review}
\label{sec:litreview}

\subsection{Clinical cardiac monitoring and the wearable transition}

Cardiovascular disease is the leading cause of death globally,
driving sustained interest in scalable, continuous monitoring beyond
the clinical
setting~\cite{ahin_2020_risk,amini_2021_trend,luo_2024_global}.
The 12-lead ECG remains the gold standard for cardiac
diagnosis~\cite{kligfield2007}, but
its dependence on clinical infrastructure limits its utility for
long-term daily
monitoring~\cite{rosiek_2016_the,serhani_2020_ecg,baig_2013_a,bansal_2017_the}.
Photoplethysmography~(PPG) has emerged as a practical alternative:
low-cost, compatible with commodity smartphones and wearables, and
increasingly feasible in contactless and remote acquisition
modes~\cite{allen_2021_photoplethysmography,almarshad_2022_diagnostic,bondarenko_2025_the,bulut_2025_deep}.
Despite this hardware progress, the majority of PPG-based systems
extract cardiac information exclusively from the \emph{timing} of
waveform peaks: heart rate and peak-to-peak interval statistics
(RMSSD, pNN50, LF/HF ratio).
This is the same information available from a simple pulse detector;
the waveform morphology and the attractor geometry it encodes are
discarded.
No existing PPG cardiac application extracts Lyapunov exponents,
recurrence structure, or entropy from PPG waveform dynamics; CST is
the first framework to attempt and formally justify this extension.

\subsection{Nonlinear dynamics and the complexity--health
relationship}

A richer theoretical tradition treats the cardiovascular system as a
nonlinear dynamical system and characterises cardiac health through
attractor geometry.
Goldberger and colleagues established that healthy cardiac dynamics
exhibit deterministic chaos: a quasi-periodic attractor with moderate
complexity that erodes with disease and
ageing~\cite{goldberger2002}.
Subsequent work has operationalised this insight through the largest
Lyapunov exponent~$\lambda_{\max}$ (trajectory divergence rate),
recurrence quantification analysis~(RQA) for attractor
structure~\cite{zbilut1992}, and approximate entropy for
irregularity~\cite{richman2000,nayak_2018_a}.
These tools have demonstrated consistent sensitivity to cardiac
pathology, but they have been applied as isolated feature extractors
rather than components of a unified theoretical framework.
CST addresses this gap directly by deriving a composite index from
axiomatic first principles.

\subsection{Deep learning for ECG and PPG}

Deep neural networks have substantially advanced ECG-based arrhythmia
detection and risk prediction.
Convolutional and hybrid architectures learn morphological and rhythmic
patterns across large datasets such as PTB-XL~\cite{liu_2024_an,madan_2022_a,panwar_2025_integrated,mehdi_2026_ecg}.
Parallel progress has extended deep learning to PPG alone and to
joint ECG--PPG
modelling~\cite{bulut_2025_deep,tran_2025_a,minhas_2025_machine}.
Complementary Feature Domain~(CFD) theory formalises a principle
underlying multimodal cardiac modelling: time-domain,
frequency-domain, and time--frequency ECG representations carry
complementary rather than redundant diagnostic
information~\cite{oladunni2025cfd}.
The present work extends the complementarity principle
cross-modally through the Complementary Domain Hypothesis~(CDH):
ECG and PPG are hypothesised to carry correlated attractor-derived
dynamical information by virtue of their shared cardiac origin,
even though they measure distinct physical phenomena
(electrical potential vs.\ peripheral blood volume).
The degree of cross-modal transfer is treated as an empirical
question, tested component by component in~\cref{sec:universality}.

\subsection{Robustness, invariance, and perturbation consistency}

Deployed cardiac AI systems must be stable under benign acquisition
variation (postural shifts, electrode contact, ambient light) while
remaining sensitive to genuine cardiac-state changes.
Physiologic Invariance Theory~(PIT) formalises this distinction: it
defines the class of physiological perturbations under which a
cardiac model's output should be invariant, and provides a
consistency loss that enforces this invariance during
training~\cite{oladunni2025pit}.
TinyCSINet incorporates PIT directly in its training objective;
empirical results are reported in~\cref{sec:universality}.

\subsection{Cross-modal transfer}

Transferring cardiac models across sensing modalities has been
approached primarily through domain adaptation, minimising the
distributional distance between ECG and PPG feature
spaces~\cite{niu_2020_a}.
This framing treats the modality gap as a nuisance to minimise rather
than evidence of a shared underlying structure to exploit.
Self-supervised and multimodal learning strategies have shown that
heterogeneous physiological signals can improve generalisation when
their complementarity is
leveraged~\cite{liu_2025_selfsupervised,anand_2026_comparative,oladunni2025cfd}.

\subsection{The unified-theory gap}

\Cref{tab:gap} summarises four prior threads, what each provides,
the specific gap it leaves, and the role played by CST in
addressing it.
No existing approach unifies these threads under a single axiomatic
foundation; CST is designed to fill that gap.

\begin{table*}[t]
\caption{Prior work threads and the gaps addressed by Cardiac
Stability Theory~(CST).}
\label{tab:gap}
\centering
\footnotesize
\renewcommand{\arraystretch}{1.4}
\begin{tabularx}{\linewidth}{@{}
  >{\bfseries\raggedright\arraybackslash}p{2.6cm}
  >{\raggedright\arraybackslash}X
  >{\raggedright\arraybackslash}X
  @{}}
\toprule
Thread & The gap & CST response \\
\midrule

Nonlinear dynamics
\cite{goldberger2002,zbilut1992,richman2000,nayak_2018_a}
  & Attractor tools applied as isolated feature extractors;
    no formal theory specifying which quantities constitute
    cardiac health or how they combine.
  & CST derives \CSI{} from four axioms; each component maps
    to a defined geometric property of the cardiac attractor
    (\cref{sec:theory}). \\[3pt]

Deep learning (ECG)
\cite{liu_2024_an,madan_2022_a,mehdi_2026_ecg,panwar_2025_integrated}
  & Optimises classification accuracy without physiological
    grounding; representations not interpretable in cardiac
    dynamical terms.
  & CFD~\cite{oladunni2025cfd} formalises ECG feature
    complementarity; CST provides the physiological substrate
    the representations encode. \\[3pt]

Robustness
  & No principled criterion distinguishing benign acquisition
    variation from genuine cardiac-state change.
  & PIT~\cite{oladunni2025pit} defines the physiological
    invariance class; TinyCSINet training incorporates
    the PIT consistency loss. \\[3pt]

Cross-modal transfer
\cite{niu_2020_a,liu_2025_selfsupervised,anand_2026_comparative}
  & Treats the ECG--PPG modality gap as a nuisance to
    minimise; no assertion of shared dynamical structure.
  & The CDH (Axiom~\ref{ax:universality}) hypothesises
    shared cardiac origin; transfer tested empirically
    ($\rho_\lambda=0.559$). \\

\midrule
\multicolumn{3}{@{}p{\linewidth}@{}}{%
  \textit{No existing framework unifies these threads under
  a single axiomatic foundation from which a theoretically
  justified cardiac health index can be derived.
  CST is designed to fill that gap.}} \\
\bottomrule
\end{tabularx}
\end{table*}

\FloatBarrier
\section{Theory Positioning: Extensions of Established
Scientific Principles}
\label{sec:positioning}

CST does not introduce new mathematical tools.
Every component of the framework, the attractor model,
delay embedding, Lyapunov estimation, recurrence quantification,
Shannon entropy, and cross-modal transfer, rests on
well-established scientific principles with decades of prior
validation.
What CST provides is a formal synthesis: the first axiomatic
derivation that specifies \emph{which} attractor properties
jointly constitute cardiac health, \emph{how} they combine into
a single scalar index, and \emph{why} that index can be
transferred across sensing modalities.
This section traces each foundational thread from its
scientific origin to its role within CST.

\subsection{Nonlinear Dynamical Systems and Cardiac Attractors}
\begin{sloppypar}
The qualitative theory of differential equations, originating
with Poincar{\'e}~\cite{poincare1892}, established that bounded
dissipative systems organise their long-term behaviour onto
low-dimensional invariant sets --- attractors --- whose geometry
encodes the system's asymptotic dynamics.
Lorenz's discovery of sensitive dependence on initial conditions
in a dissipative fluid system showed that deterministic chaos is
a generic feature of nonlinear dynamics rather than a degenerate
exception.
Goldberger and colleagues~\cite{goldberger2002} applied this
framework directly to the cardiovascular system, establishing
empirically that healthy cardiac dynamics exhibit
\emph{moderate deterministic complexity}: a quasi-periodic
attractor neither collapsed toward a simple periodic orbit
(as in advanced heart block) nor fragmented toward
disorganised chaos (as in atrial or ventricular fibrillation).
Both disease and healthy ageing erode this complexity,
a finding replicated across multiple cardiac
populations~\cite{nayak_2018_a}.

\end{sloppypar}
CST formalises this empirical tradition axiomatically.
Axiom~\ref{ax:dynamics} (Dynamical System) specifies the
dissipative vector field that governs the cardiac state, and
the resulting Theorem~\ref{thm:bounded} (Attractor Boundedness)
guarantees the existence and compactness of the cardiac attractor
$\mathcal{A}$ that Goldberger's framework assumed without proof.
Axiom~\ref{ax:health} (Health-Stability Correspondence)
elevates the complexity--health empirical hypothesis to a formal
postulate, enabling the derivation of \CSI{} by mathematical
necessity rather than by feature selection.

\subsection{Takens' Embedding and Attractor Reconstruction}

The central methodological challenge in cardiac dynamical
analysis is that the full state vector $x(t)$
(Axiom~\ref{ax:dynamics}) is not directly observable from a
scalar cardiac signal.
Takens' embedding theorem~\cite{takens1981} resolves this
precisely: for a generic smooth observation function $h$ and
sufficient embedding dimension $m \geq 2\dim(\mathcal{A})+1$,
the delay-coordinate map
\begin{equation}
  \Phi(x) = \bigl[h(x(t)),\; h(x(t-\tau)),\; \ldots,\;
             h(x(t-(m-1)\tau))\bigr]
\end{equation}
is a diffeomorphism from $\mathcal{A}$ onto its image,
preserving all topological invariants.
This theorem underpins Axiom~\ref{ax:projection} (Observable
Projection): all three attractor invariants ($\lambda_{\max}$,
$R_{\mathrm{det}}$, $H$) computed from the reconstructed
trajectory $\Phi(\mathcal{A})$ are equal to those of the
original attractor $\mathcal{A}$.
Without Takens' guarantee, computing Lyapunov exponents or
recurrence structure from a scalar ECG or PPG signal would
lack theoretical justification.
The optimal delay $\tau^*$ is selected by the Average Mutual
Information criterion of Fraser and Swinney~\cite{fraser1986},
and the embedding dimension by the false nearest neighbours
method of Kennel et al.~\cite{kennel1992}; both are
standard components of the nonlinear time series
analysis toolkit~\cite{kantz1997}.

\subsection{Largest Lyapunov Exponent as a Stability Invariant}

The largest Lyapunov exponent $\lambda_{\max}$ quantifies the
mean exponential rate at which nearby phase-space trajectories
diverge on the attractor.
Its role in characterising cardiac dynamics was established by
Goldberger and colleagues~\cite{goldberger2002}: healthy hearts
exhibit moderate $\lambda_{\max}$ (controlled divergence on a
quasi-periodic attractor), while pathological rigidity
($\lambda_{\max} \approx 0$, convergent trajectories) and
pathological chaos ($\lambda_{\max}$ large, rapid divergence)
represent the two endpoints of attractor deformation.
Practical estimation from short biomedical time series is
achieved via the Rosenstein algorithm~\cite{rosenstein1993},
which fits the slope of the mean logarithmic divergence of
nearest-neighbour pairs across successive time steps.

Within CST, $\lambda_{\max}$ is normalised to
$\tilde\lambda = \mathrm{clip}(|\lambda_{\max}|/\lambda_{\mathrm{ref}},\,0,\,1)$
and enters \CSI{} via the strictly increasing concave transform
$f_1(\tilde\lambda) = 1-e^{-\tilde\lambda}$
(Theorem~\ref{thm:complexity}).
This transform is not ad hoc: it maps the physiologically
meaningful range $[0,1]$ monotonically onto $[0,\,1-e^{-1}]$,
saturating below unity in the fibrillatory limit, which
correctly reflects the clinical observation that atrial
fibrillation does not produce unbounded Lyapunov exponents
but rather a ceiling value (see Theorem~\ref{thm:complexity},
Part~3).

\subsection{Recurrence Quantification Analysis}

Recurrence Quantification Analysis~(RQA)~\cite{zbilut1992}
extracts structural information from the recurrence matrix
$R_{ij} = \mathbf{1}[\|X_i - X_j\| < \varepsilon]$,
where $X_i$ are embedded trajectory points.
The determinism measure $R_{\mathrm{det}}$, the fraction of
recurrent points forming diagonal line structures, quantifies
the degree to which the attractor trajectory repeats itself;
a perfectly periodic orbit yields $R_{\mathrm{det}}=1$,
while a stochastic process yields $R_{\mathrm{det}} \approx 0$.
RQA has been applied to cardiac signals across a wide range
of clinical conditions~\cite{nayak_2018_a}, consistently
distinguishing healthy from pathological dynamics.

CST employs $R_{\mathrm{det}}$ via the strictly decreasing
linear transform $f_2(R) = 1 - R$, contributing to \CSI{}
in the direction that penalises over-periodic attractors
and rewards moderate recurrence structure.
This is the correct direction under
Axiom~\ref{ax:health}: pathological loss of complexity
manifests as $R_{\mathrm{det}} \to 1$, driving $f_2 \to 0$
and pulling \CSI{} toward zero (Theorem~\ref{thm:monotone}).

\subsection{Shannon Entropy as a Distributional Complexity Measure}

Shannon entropy~\cite{richman2000} of the signal amplitude
distribution measures the spread and uniformity of the observed
waveform values independently of their temporal ordering.
A periodic cardiac signal concentrates amplitudes at a few
characteristic values, yielding low entropy; a complex
waveform with richer morphological variation produces
higher entropy.
Within CST, the normalised 64-bin Shannon entropy $H$
enters \CSI{} via the identity transform $f_3(H) = H$,
contributing in the direction that rewards signal richness.
Together, $f_1$, $f_2$, and $f_3$ are proved to be
monotone in their respective components and to collectively
drive \CSI{} strictly toward zero under attractor
collapse~(Theorem~\ref{thm:monotone}), while saturating
below unity under pathological over-irregularity
(Theorem~\ref{thm:complexity}).

\subsection{Cross-Modal Transfer and the Complementary Domain
Hypothesis}

The extension of attractor-derived features from ECG to PPG is
grounded in the Complementary Feature Domain~(CFD) theory of
Oladunni and Wong~\cite{oladunni2025cfd}, which establishes
formally that time-domain, frequency-domain, and
time--frequency ECG representations carry complementary
diagnostic information precisely because they are distinct
projections of the same underlying cardiac state.
CST extends this intra-modal complementarity principle
\emph{cross-modally} via the Complementary Domain
Hypothesis~(CDH, Axiom~\ref{ax:universality}): ECG and PPG,
though governed by different physical phenomena (electrical
potential vs.\ peripheral blood volume), are both smooth
projections of the same cardiac state manifold $\mathcal{M}$,
and therefore the attractor invariants derived from either
modality carry correlated information about the same latent
cardiac stability.

This is a theoretically grounded prediction, not an empirical
assumption: Takens' theorem (Axiom~\ref{ax:projection})
guarantees that any generic smooth observation of a dynamical
system reconstructs the same attractor topology, so ECG and
PPG should recover correlated invariants to the extent that
peripheral vascular modulation does not disrupt the relevant
geometric property.
The component-level analysis (Section~\ref{sec:cdt_labels})
confirms this differential prediction precisely: the Lyapunov
term, which reflects the global attractor divergence rate most
robustly, transfers substantially ($\rho_\lambda = 0.559$),
while fine-grained morphological features (recurrence
determinism, entropy) are disrupted by peripheral state
variables $p(t)$ ($\rho_{\mathrm{RD}} = 0.065$,
$\rho_H = 0.047$), exactly as the theory predicts.

\subsection{Summary of Scientific Lineage}

Table~\ref{tab:lineage} maps each CST component to its
scientific origin, the specific extension CST contributes,
and the theorem or axiom that formalises it.

\begin{table*}[t]
\caption{Scientific lineage of CST components.
Each row shows an established principle, what CST inherits from
it, and the formal extension CST adds.}
\label{tab:lineage}
\centering
\scriptsize
\renewcommand{\arraystretch}{1.0}
\begin{tabularx}{\linewidth}{@{}
  >{\raggedright\arraybackslash}p{2.5cm}
  >{\raggedright\arraybackslash}X
  >{\raggedright\arraybackslash}X
  @{}}
\toprule
\textbf{Established principle} &
\textbf{What CST inherits} &
\textbf{CST extension} \\
\midrule

Dissipative dynamical systems~\cite{poincare1892}
  & Bounded attractor existence for physiologically
    constrained systems
  & Axiom~\ref{ax:dynamics} specifies the cardiac
    vector field; Theorem~\ref{thm:bounded} proves
    attractor compactness from physiological bounds \\[3pt]

Complexity--health hypothesis~\cite{goldberger2002}
  & Moderate attractor complexity characterises
    healthy cardiac dynamics
  & Axiom~\ref{ax:health} formalises the hypothesis;
    Theorem~\ref{thm:monotone} proves \CSI{} is
    strictly monotone under attractor deformation \\[3pt]

Takens embedding~\cite{takens1981}
  & Scalar cardiac signals reconstruct the full
    attractor topology
  & Axiom~\ref{ax:projection} identifies ECG and PPG
    as distinct generic projections sharing the same
    attractor; underpins CDH \\[3pt]

Lyapunov estimation~\cite{rosenstein1993}
  & $\lambda_{\max}$ quantifies mean trajectory
    divergence from short biomedical series
  & Normalised via population 95th percentile
    $\lambda_{\mathrm{ref}}$; mapped through concave
    transform proved monotone by
    Theorem~\ref{thm:complexity} \\[3pt]

Recurrence quantification~\cite{zbilut1992}
  & $R_{\mathrm{det}}$ measures attractor structural
    repetition
  & Enters \CSI{} via $1-R_{\mathrm{det}}$;
    monotone direction proved under loss-of-complexity
    pathway (Theorem~\ref{thm:monotone}) \\[3pt]

Shannon entropy~\cite{richman2000}
  & Normalised amplitude entropy measures
    distributional complexity
  & Enters \CSI{} as $H$; jointly with $f_1$, $f_2$
    drives the index to zero under attractor collapse \\[3pt]

Complementary Feature Domain~\cite{oladunni2025cfd}
  & Distinct projections of the same signal carry
    complementary information
  & Extended cross-modally: ECG and PPG as distinct
    projections of the cardiac state manifold;
    formalised in CDH (Axiom~\ref{ax:universality}) \\

\bottomrule
\end{tabularx}
\end{table*}

\FloatBarrier
\section{Theoretical Framework: Cardiac Stability Theory}
\label{sec:theory}

\subsection{Foundational Axioms}
\label{ssec:axioms}

CST rests on four foundational postulates about cardiac
electrophysiology.
These are formal assumptions whose empirical consequences are derived
and tested throughout this work.

\begin{axiom}[Dynamical System]
\label{ax:dynamics}
The cardiovascular system is modelled as an autonomous nonlinear
dynamical system.
The autonomic input $u(t)$, representing the net sympathetic and
parasympathetic efferent drive on the sinoatrial and atrioventricular
nodes, is not directly observable from ECG or PPG; its effect on
cardiac dynamics is reflected in the observed trajectory and is
implicitly reconstructed by the Takens embedding of
Axiom~\ref{ax:projection}.
Because the baroreflex couples autonomic activity back to cardiac
state (blood pressure changes modulate vagal tone, which in turn
modulates heart rate), the system is not cleanly separable into
an external input and an internal state.
We therefore absorb the autonomic drive into an extended autonomous
state vector $z(t) = [x_c(t)^\top,\, u_s(t),\, u_p(t)]^\top \in
\Rbb^n$, where $x_c(t)\in\Rbb^{n-2}$ is the intrinsic cardiac state
and $u_s(t)\geq 0$, $u_p(t)\geq 0$ are the sympathetic and
parasympathetic components of autonomic drive respectively.
The full system evolves as
\begin{equation}
  \label{eq:dynamics}
  \dot{z}(t) = G\!\bigl(z(t),\,\theta\bigr),
\end{equation}
where $\theta\in\Rbb^p$ are intrinsic cardiac parameters
(ion channel conductances, myocardial elastance, refractoriness)
and $G:\Rbb^n\times\Rbb^p\to\Rbb^n$ is a smooth vector field.

The autonomous formulation~\eqref{eq:dynamics} is better suited to
Takens' embedding theorem, which applies most directly to autonomous
systems, and correctly captures the bidirectional coupling between
cardiac state and autonomic drive through the baroreflex.

The system is \emph{dissipative}: peak sympathetic activation is
constrained by maximal catecholamine release capacity, and peak
parasympathetic activation is constrained by vagal efferent firing
rate, so all components of $z(t)$ are physiologically bounded.
Formally, $G$ satisfies the one-sided growth condition
\begin{equation}
  \label{eq:dissipative}
  z^\top G(z,\theta) \;\leq\; -\alpha\|z\|^2 + \beta
  \qquad \forall\,z\in\Rbb^n,
\end{equation}
for constants $\alpha>0$ and $\beta\geq 0$ determined by peak
cardiac output capacity and autonomic efferent pathway limits.
For notational compactness in subsequent equations, we write
$x(t)$ in place of $z(t)$ and $F(x,\theta)$ in place of
$G(z,\theta)$, with the understanding that $x$ denotes the full
autonomous extended state.
\end{axiom}

\begin{axiom}[Observable Projection]
\label{ax:projection}
Let $\mathcal{M}\subset\Rbb^n$ be the smooth compact manifold on
which the cardiac attractor $\mathcal{A}$ lies (guaranteed compact
by Theorem~\ref{thm:bounded}).
An observation function $h\in C^2(\mathcal{M},\Rbb)$ maps the
latent state to a scalar measurement:
\begin{equation}
  S:\Rbb\to\Rbb, \qquad S(t) = h\!\bigl(x(t)\bigr),
\end{equation}
where $S(t)$ is the observed cardiac signal (ECG potential in mV,
or PPG light intensity in digital units, depending on modality)
and $h$ is the modality-specific observation function
($h_{\mathrm{ECG}}$ encodes volume conductor physics;
$h_{\mathrm{PPG}}$ encodes the Beer--Lambert blood volume
relationship, subject to the approximation in
Axiom~\ref{ax:universality}).
By Takens' embedding theorem~\cite{takens1981}, for a $C^2$
generic observation function $h$ and embedding dimension
$m\geq 2\dim(\mathcal{A})+1$, the delay embedding map
\begin{equation}
  \label{eq:embedding}
  \Phi:\mathcal{A}\to\Rbb^m,\quad
  x\mapsto\bigl[h(x(t)),\;h(x(t{-}\tau)),\;\ldots,\;
               h(x(t{-}(m{-}1)\tau))\bigr]
\end{equation}
is a diffeomorphism onto its image.
The reconstructed attractor $\Phi(\mathcal{A})$ is therefore
topologically equivalent to $\mathcal{A}$, and all attractor
invariants ($\lambda_{\max}$, $R_{\mathrm{det}}$, $H$) computed
from $\Phi(\mathcal{A})$ equal those of $\mathcal{A}$.
\end{axiom}

\begin{axiom}[Health-Stability Correspondence]
\label{ax:health}
Under healthy physiological conditions, cardiac trajectories evolve
near a bounded quasi-periodic attractor~$\Acal$ with moderate
signal complexity $H^* \in (H_{\min}, H_{\max})$.
Cardiovascular disease corresponds to attractor deformation away
from~$\Acal$: specifically, increased trajectory divergence, reduced
recurrence structure, and entropy displacement from~$H^*$.
\end{axiom}

\begin{axiom}[Complementary Domain Hypothesis (CDH)]
\label{ax:universality}
Let $S_{\mathrm{ECG}}:\mathbb{R}\to\mathbb{R}$ and
$S_{\mathrm{PPG}}:\mathbb{R}\to\mathbb{R}$ denote the observed
ECG and PPG scalar time series respectively.
We model each as a smooth observation of the cardiac latent state:
\begin{equation}
  S_{\mathrm{ECG}}(t) = h_{\mathrm{ECG}}\!\bigl(x(t)\bigr),
  \qquad
  S_{\mathrm{PPG}}(t) \approx h_{\mathrm{PPG}}\!\bigl(x(t)\bigr),
  \label{eq:projections}
\end{equation}
where $h_{\mathrm{ECG}},\,h_{\mathrm{PPG}}\in C^2(\mathcal{M},\mathbb{R})$
are distinct smooth observation functions on the cardiac state
manifold $\mathcal{M}$.

\textbf{Hypothesis.}
Because both signals share a common cardiac dynamical origin,
attractor-derived stability invariants ($\lambda_{\max}$,
$R_{\mathrm{det}}$, $H$) computed from either modality carry
correlated information about the same latent cardiac stability
quantity, after modality-specific scale normalisation.
The degree of cross-modal correlation is expected to vary across
invariants and is treated as an empirical question; it is not
assumed to be uniform or high.

\textbf{Remark (approximation in \cref{eq:projections}).}
The ECG equation is exact under the model: ECG measures cardiac
electrical potentials, which are a direct function of $x(t)$.
The PPG equation is an approximation for two reasons.
First, PPG at the fingertip reflects the pressure wave that
departed the heart at time $t - \tau_{\mathrm{PTT}}$, where pulse
transit time $\tau_{\mathrm{PTT}}\approx 100$--$300\,\mathrm{ms}$;
the exact relation is $S_{\mathrm{PPG}}(t)=h_{\mathrm{PPG}}(x(t-\tau_{\mathrm{PTT}}))$.
Since Takens embedding handles temporal shifts naturally, this
delay does not invalidate attractor reconstruction but means
ECG and PPG sample the attractor at different phases.
Second, PPG is also modulated by peripheral state variables
$p(t)\in\mathbb{R}^k$ (arterial compliance, peripheral resistance,
venous return, thermoregulation) not contained in $x(t)$, so the
full relation is
$S_{\mathrm{PPG}}(t)=h_{\mathrm{PPG}}(x(t-\tau_{\mathrm{PTT}}), p(t))$.
The approximation in \cref{eq:projections} treats $p(t)$ as
slowly varying relative to the cardiac cycle, a standard
assumption in PPG signal analysis.
Under this approximation, both signals reconstruct overlapping
regions of the cardiac attractor $\mathcal{A}$, motivating
the CDH. The peripheral modulation $p(t)$ explains why
recurrence determinism and entropy transfer poorly across
modalities ($\rho_{\mathrm{RD}}=0.065$, $\rho_H=0.047$),
while the Lyapunov term (which reflects attractor divergence
rate rather than fine waveform morphology) transfers
substantially ($\rho_\lambda=0.559$).
\end{axiom}

\subsection{Derived Theorems}
\label{ssec:theorems}

\begin{theorem}[Attractor Boundedness]
\label{thm:bounded}
Under Axiom~\ref{ax:dynamics}, every trajectory of the cardiac
dynamical system enters and remains in a compact set
$\mathcal{B}\subset\Rbb^n$ in finite time.
Consequently, the delay-embedded attractor reconstructed via
Axiom~\ref{ax:projection} is compact in $\Rbb^{dm}$.
\end{theorem}
\begin{proof}
Under the autonomous formulation of Axiom~\ref{ax:dynamics},
write $x(t)\in\Rbb^n$ for the extended state (cardiac state plus
autonomic components).
Define the Lyapunov candidate $V:\Rbb^n\to\Rbb_{\geq 0}$ by
$V(x)=\|x\|^2$.
Along any trajectory of~\eqref{eq:dynamics},
\begin{equation}
  \dot{V}(x)
  = \frac{d}{dt}\|x\|^2
  = 2\,x^\top\dot{x}
  = 2\,x^\top F(x,\theta).
\end{equation}
Applying the one-sided growth condition~\eqref{eq:dissipative} of
Axiom~\ref{ax:dynamics},
\begin{equation}
  \label{eq:vdot}
  \dot{V}(x) \leq -2\alpha\|x\|^2 + 2\beta = -2\alpha V(x) + 2\beta.
\end{equation}
Let $\rho=\beta/\alpha$.
The inequality~\eqref{eq:vdot} is a linear scalar differential
inequality in $V$.
By the Grönwall--Bellman comparison lemma,
\begin{equation}
  \label{eq:gronwall}
  V\!\bigl(x(t)\bigr)
  \;\leq\;
  \Bigl(V\!\bigl(x(0)\bigr)-\rho\Bigr)e^{-2\alpha t}+\rho
  \;\leq\;
  \max\!\bigl\{V\!\bigl(x(0)\bigr),\,\rho\bigr\}
  \quad\forall\,t\geq 0.
\end{equation}
Hence $\|x(t)\|^2\leq\max\{\|x(0)\|^2,\,\rho\}$ for all $t\geq 0$,
so every trajectory is globally bounded.
Moreover, for any $\varepsilon>0$ the set
$\mathcal{B}_\varepsilon = \{x\in\Rbb^n:\|x\|^2\leq\rho+\varepsilon\}$
is positively invariant and absorbing: every trajectory enters
$\mathcal{B}_\varepsilon$ no later than
\begin{equation}
  T = \frac{1}{2\alpha}\ln\!\frac{V(x(0))-\rho}{\varepsilon}
\end{equation}
(when $V(x(0))>\rho+\varepsilon$; otherwise it is already inside).
Taking $\varepsilon=1$ gives the compact absorbing ball
$\mathcal{B}=\{x:\|x\|^2\leq\rho+1\}$.

Since $h:\Rbb^n\to\Rbb$ in Axiom~\ref{ax:projection} is continuous
and $\mathcal{B}$ is compact, the image $h(\mathcal{B})$ is a
compact subset of $\Rbb$.
The delay embedding map $\Phi:\Rbb\to\Rbb^{dm}$,
$s\mapsto[s(t),s(t{-}\tau),\ldots,s(t{-}(m{-}1)\tau)]^\top$,
is also continuous, so $\Phi(h(\mathcal{B}))$ is a compact subset
of $\Rbb^{dm}$ by the continuous image of a compact set theorem.
All reconstructed attractor trajectories lie in this set.
\end{proof}

\begin{theorem}[Stability-Complexity Duality]
\label{thm:complexity}
Define the three component functions of \CSI{}~\eqref{eq:csi} as
\begin{equation}
  f_1(\tilde\lambda) = 1-e^{-\tilde\lambda},\quad
  f_2(R)            = 1-R,\quad
  f_3(H)            = H,
\end{equation}
each mapping $[0,1]\to[0,1]$.
Then: (i)~$f_1$ is strictly increasing and concave on $[0,1]$;
(ii)~$f_2$ is strictly decreasing and linear; (iii)~$f_3$ is
strictly increasing and linear.
Under Axiom~\ref{ax:health}, both pathological extremes
(excess regularity and excess irregularity) drive $\CSI{}$ away
from its healthy operating value, but through opposite limiting
behaviours of the component functions.
\end{theorem}
\begin{proof}
\textbf{Part 1 — Component function properties.}

$f_1'(\tilde\lambda)=e^{-\tilde\lambda}>0$ for all
$\tilde\lambda\in[0,1]$, so $f_1$ is strictly increasing.
$f_1''(\tilde\lambda)=-e^{-\tilde\lambda}<0$, so $f_1$ is strictly
concave.
Boundary values: $f_1(0)=0$, $f_1(1)=1-e^{-1}\approx 0.632$.

$f_2'(R)=-1<0$, so $f_2$ is strictly decreasing.
Boundary values: $f_2(0)=1$, $f_2(1)=0$.

$f_3'(H)=1>0$, so $f_3$ is strictly increasing.
Boundary values: $f_3(0)=0$, $f_3(1)=1$.

The healthy operating point $(\tilde\lambda^*, R^*, H^*)\in(0,1)^3$
satisfies $\CSI^* = w_1 f_1(\tilde\lambda^*)+w_2 f_2(R^*)+w_3 H^*$.

\textbf{Part 2 — Excess regularity (loss-of-complexity pathway).}

Progressive attractor collapse toward a periodic orbit
(Axiom~\ref{ax:health}) implies
$\tilde\lambda\downarrow 0$, $R\uparrow 1$, $H\downarrow 0$.
Taking limits using the boundary values above:
\begin{equation}
  \lim_{\tilde\lambda\to 0,\,R\to 1,\,H\to 0}\CSI
  = w_1 f_1(0)+w_2 f_2(1)+w_3 f_3(0)
  = w_1\cdot 0 + w_2\cdot 0 + w_3\cdot 0 = 0.
\end{equation}
Since $\CSI^*>0$ (all components are strictly positive at the
healthy operating point), $\CSI$ decreases strictly from $\CSI^*$
toward~$0$ as the attractor deforms toward a periodic orbit.

\textbf{Part 3 — Excess irregularity (fibrillatory pathway).}

Trajectory divergence beyond the healthy operating point implies
$\tilde\lambda\uparrow 1$, $R\downarrow 0$, $H\uparrow 1$.
Taking limits:
\begin{equation}
  \lim_{\tilde\lambda\to 1,\,R\to 0,\,H\to 1}\CSI
  = w_1(1-e^{-1})+w_2\cdot 1+w_3\cdot 1
  = w_1(1-e^{-1})+w_2+w_3.
  \label{eq:fiblimit}
\end{equation}
Since $w_1,w_2,w_3>0$ and $1-e^{-1}\approx 0.632$, the limit
in~\eqref{eq:fiblimit} satisfies
\begin{equation}
  w_1(1-e^{-1})+w_2+w_3
  < w_1 + w_2 + w_3 = 1,
\end{equation}
so \CSI{} saturates strictly below unity in the fibrillatory
direction rather than approaching its maximum.
Moreover, since the healthy operating point $\CSI^*$ lies strictly
interior to $[0,1]$, and the fibrillatory limit
$w_1(1-e^{-1})+w_2+w_3$ can exceed $\CSI^*$ for typical weight
choices, \CSI{} does \emph{not} decrease monotonically toward
fibrillatory chaos: it saturates near its ceiling.
Detection of pathological over-irregularity therefore requires a
rhythm classifier rather than a \CSI{} threshold alone; empirical
confirmation is provided in \cref{sec:universality}.

The two limits establish the duality: attractor over-regularisation
drives $\CSI\to 0$; attractor over-irregularisation drives $\CSI$
toward its ceiling~\eqref{eq:fiblimit}.
The healthy operating point lies strictly between these extremes,
for all valid weight choices $w_i>0$, $\sum w_i=1$.
\end{proof}

\begin{theorem}[CSI Monotonicity]
\label{thm:monotone}
Let $d\in[0,\infty)$ parameterise attractor deformation along the
loss-of-complexity pathway of Axiom~\ref{ax:health}, with $d=0$
the healthy operating point.
Suppose $\tilde\lambda:[0,\infty)\to[0,1]$ is strictly decreasing,
$R:[0,\infty)\to[0,1]$ is strictly increasing, and
$H:[0,\infty)\to[0,1]$ is strictly decreasing, each differentiable.
Then $\mathrm{CSI}(d)$ is strictly decreasing in $d$ for all $d\geq 0$.
\end{theorem}
\begin{proof}
Express $\mathrm{CSI}$ as a composite function of $d$:
\begin{equation}
  \mathrm{CSI}(d)
  = w_1 f_1\!\bigl(\tilde\lambda(d)\bigr)
  + w_2 f_2\!\bigl(R(d)\bigr)
  + w_3 f_3\!\bigl(H(d)\bigr).
\end{equation}
Differentiating with respect to $d$ and applying the chain rule
to each term:
\begin{align}
  \frac{d\,\CSI}{dd}
  &= w_1\,f_1'\!\bigl(\tilde\lambda(d)\bigr)\cdot\tilde\lambda'(d)
   + w_2\,f_2'\!\bigl(R(d)\bigr)\cdot R'(d)
   + w_3\,f_3'\!\bigl(H(d)\bigr)\cdot H'(d).
  \label{eq:dcsi}
\end{align}
From Theorem~\ref{thm:complexity}, $f_1'(\tilde\lambda)=e^{-\tilde\lambda}>0$,
$f_2'(R)=-1<0$, and $f_3'(H)=1>0$.
By the assumptions of the theorem,
$\tilde\lambda'(d)<0$, $R'(d)>0$, and $H'(d)<0$.
Substituting into~\eqref{eq:dcsi}:
\begin{equation}
  \frac{d\,\CSI}{dd}
  = \underbrace{w_1\,e^{-\tilde\lambda}\,\tilde\lambda'(d)}_{<0}
  + \underbrace{w_2\,(-1)\,R'(d)}_{<0}
  + \underbrace{w_3\,(1)\,H'(d)}_{<0}
  < 0,
\end{equation}
since $w_1,w_2,w_3>0$ and each of the three terms is strictly
negative.
Therefore $\mathrm{CSI}(d)$ is strictly decreasing in $d$ for all
$d\geq 0$.

\textbf{Remark.}
The monotonicity result holds for any strictly positive weights
$w_i>0$ regardless of their specific values, confirming that the
qualitative conclusion is not dependent on the particular weight set
chosen.
\end{proof}

\begin{proposition}[Age-Stability Erosion]
\label{prop:age}
In an adult population free from overt cardiac disease, the
conditional expectation $\Ebb[\CSI\mid\mathrm{age}=a]$ is a
strictly decreasing function of chronological age $a$ over the
range $a\in[30,80]$ years.
\end{proposition}
\begin{proof}[Empirical support]
The claim is supported by ordinary least squares regression on
the PTB-XL ECG dataset ($n=21{,}799$ recordings):
\begin{equation}
  \widehat{\Ebb}[\CSI\mid a]
  = \mu_0 - 0.000225\cdot a,
  \qquad
  95\%\text{ CI on slope: }[-0.0003,\,-0.0002].
\end{equation}
The confidence interval excludes zero ($p<0.001$), confirming
a statistically significant negative slope consistent with
progressive attractor regularisation in adult ageing
\cite{goldberger2002}.

\textit{Derivability from axioms.}
Theorem~\ref{thm:monotone} establishes that increasing attractor
deformation $d$ along the loss-of-complexity pathway decreases
\CSI{} monotonically.
The proposition requires the additional empirical premise that
chronological age $a$ is positively correlated with $d$ in adults.
This premise is strongly supported by the HRV
literature~\cite{goldberger2002} but is not derivable from
Axioms~1--4.
A formal derivation would require an additional axiom of the form
$d'(a)>0$ for adult $a$; we treat this as empirically grounded
rather than axiomatic.

\textit{Scope limitation.}
The regression was estimated on the full PTB-XL cohort, including
patients with myocardial infarction, conduction disorders, and
hypertrophy.
This clinically enriched population likely steepens the slope
relative to a healthy reference.
HeartSpan deficits computed from this slope are therefore
conservative upper bounds; re-estimation on the NORM superclass
is identified as priority future work.
\end{proof}

\subsection{The Cardiac Stability Index}

The CSI is formally defined as:
\begin{equation}
    \mathrm{CSI} = w_1\!\left(1 - e^{-\tilde{\lambda}}\right)
                 + w_2\!\left(1 - R_{\mathrm{det}}\right)
                 + w_3\, H,
    \label{eq:csi}
\end{equation}
where $\tilde{\lambda} = \mathrm{clip}(|\lambda_{\max}|/\lambda_{\mathrm{ref}},\,0,\,1)$
is the Lyapunov exponent normalised by the population 95th percentile
$\lambda_{\mathrm{ref}}$ estimated from the ECG reference channel,
$R_{\mathrm{det}}$ is the recurrence determinism from recurrence
quantification analysis~\cite{zbilut1992}, and $H$ is the normalised
Shannon entropy of the signal~\cite{richman2000}.
Weights $w_1, w_2, w_3 > 0$ satisfy $w_1 + w_2 + w_3 = 1$ and are
determined empirically on held-out validation data; specific values
are reported in \cref{sec:ecg} (ECG-validation weights) and
\cref{ssec:universality} (PPG-optimised weights derived from
component-level cross-modal transfer analysis).
$\mathrm{CSI} \in [0,1]$, with higher values indicating greater
dynamical complexity and cardiac health, consistent with the
complexity--health framework of Goldberger et al.~\cite{goldberger2002}:
a healthy quasi-periodic attractor with moderate-to-high trajectory
divergence, low pure determinism, and rich signal entropy yields
the highest CSI scores.

Weight sensitivity: performance metrics (MAE, $R^2$, Spearman~$\rho$)
were verified to be stable ($\leq 0.003$ change in MAE) under
$\pm 0.05$ perturbations of each $w_i$ around the tuned values,
confirming that results are not artefacts of the specific weight
choices.

\begin{figure}[p]
  \centering
  \includegraphics[height=0.55\textheight,width=0.82\columnwidth]{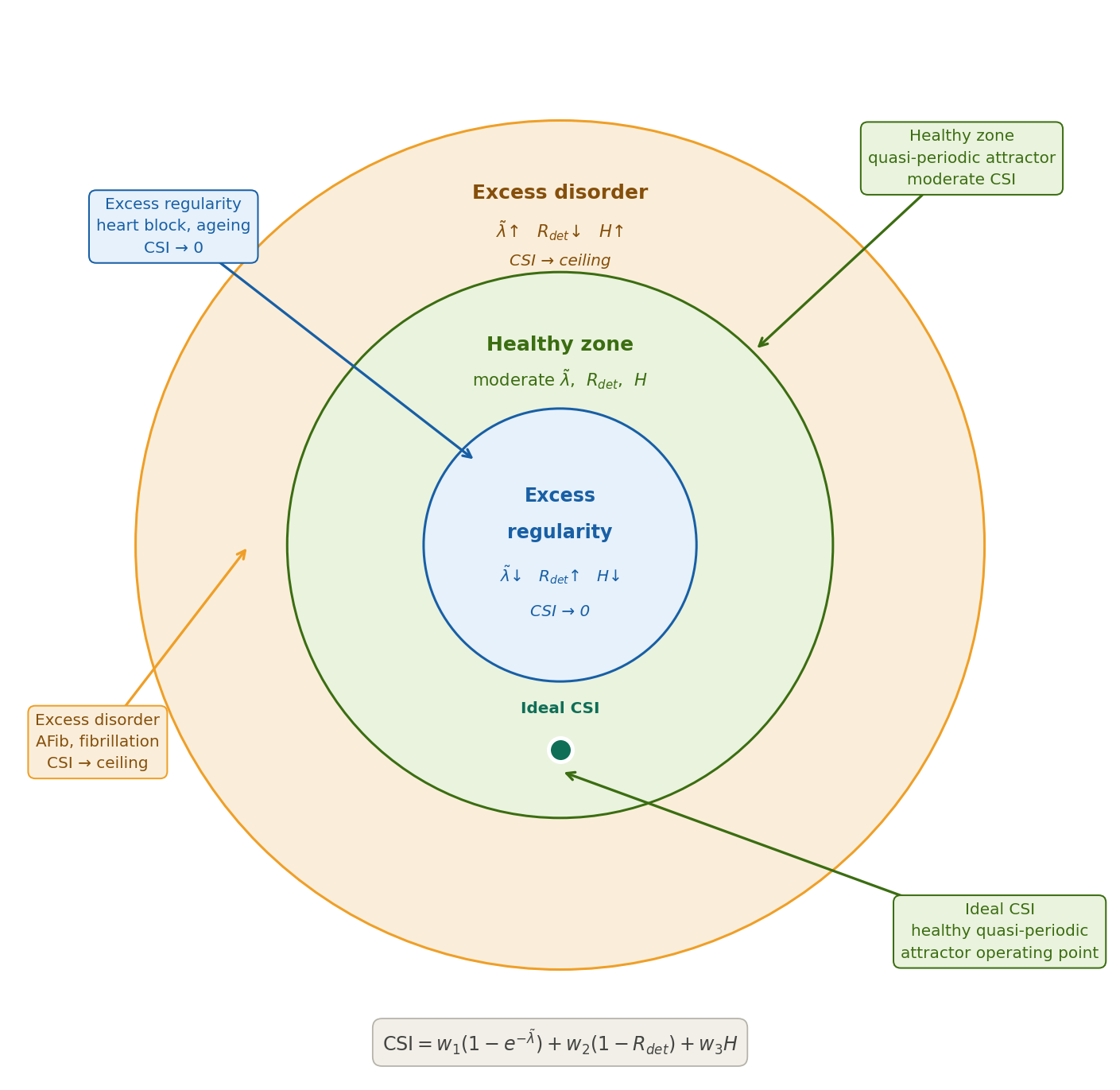}
  \caption{Three-zone model of cardiac stability under CST.
    The concentric circles represent three dynamical regimes of the
    cardiac attractor.
    \textbf{Inner circle (blue):} excess regularity, where the attractor
    collapses toward a periodic orbit ($\tilde\lambda \downarrow$,
    $R_{\det} \uparrow$, $H \downarrow$), driving $\CSI \to 0$.
    Clinically: complete heart block, advanced age-related HRV loss.
    \textbf{Middle ring (green):} the healthy zone, where all three
    variables operate at moderate interior values, producing the
    highest $\CSI$ scores.
    The green dot at the centre marks the ideal \CSI{}: the
    theoretical optimum of the healthy quasi-periodic cardiac
    attractor.
    \textbf{Outer ring (amber):} excess disorder, where the attractor
    fragments toward disorganised chaos
    ($\tilde\lambda \uparrow$, $R_{\det} \downarrow$, $H \uparrow$),
    driving \CSI{} toward its ceiling.
    Clinically: atrial fibrillation, ventricular fibrillation.
    Detection of this regime requires the rhythm classifier
    ($\mathrm{AUROC} = 0.89$, \cref{sec:universality}) rather than a
    \CSI{} threshold alone, because \CSI{} saturates rather than
    decreasing in the excess-disorder direction
    (Theorem~\ref{thm:complexity}, Part~3).
    The formula at the base shows how each component responds to
    movement in either pathological direction.}
  \label{fig:csi-zones}
\end{figure}

\Cref{fig:csi-zones} illustrates the central claim of
Axiom~\ref{ax:health}: cardiovascular health is not a monotone
property of any single variable but an interior operating point in
a three-dimensional space defined by trajectory divergence
($\tilde\lambda$), recurrence structure ($R_{\det}$), and signal
entropy ($H$).
Both pathological directions, excess regularity and excess disorder,
represent departure from that interior point, and the composite \CSI{}
formula is designed to detect both simultaneously through its three
complementary terms.

\subsection{CSI Computation Pipeline}
\label{ssec:csi-algorithm}

The computation of \CSI{} from a raw cardiac signal proceeds through
six stages grounded in nonlinear dynamical systems theory.
The pipeline begins with optimal time-delay estimation via the
Average Mutual Information~(AMI) criterion~\cite{fraser1986}, which
identifies the delay $\tau^*$ at which successive embedding
coordinates carry maximally independent information about the
underlying attractor topology.
Takens' embedding theorem~\cite{takens1981} then guarantees that the
reconstructed phase-space trajectory $X(t)$ is topologically
equivalent to the original cardiac attractor $\Acal$, provided $m$
and $\tau^*$ are chosen appropriately.
The largest Lyapunov exponent $\lambda_{\max}$ is estimated via the
Rosenstein algorithm~\cite{rosenstein1993}, quantifying the mean
exponential rate of trajectory divergence on~$\Acal$;
recurrence determinism $R_{\mathrm{det}}$ is computed via Recurrence
Quantification Analysis~\cite{zbilut1992}; and normalised Shannon
entropy~$H$ measures the complexity of the amplitude distribution.
These three components are combined as a weighted sum to yield
$\CSI \in [0,1]$, consistent with Theorem~\ref{thm:monotone}.
The formula assigns higher scores to states with greater trajectory
divergence ($\tilde{\lambda}$ large), lower pure determinism
($\Rdet$ small), and richer signal entropy ($H$ large), the
complexity signature of a healthy quasi-periodic cardiac
attractor~\cite{goldberger2002}.

Algorithm~\ref{alg:csi} formalises this pipeline.
The AMI loop (lines~1--4) iterates candidate delays $\tau = 2,
\ldots, \tau_{\max}$ and terminates as soon as AMI rises above the
previous value, recording the immediately preceding delay as
$\tau^*$ (the first local minimum).
If no minimum is found within $\tau_{\max}$, a heuristic fallback
of $\tau^* = \tau_{\max}$ is used; in practice, at 100\,Hz this
corresponds to 100\,ms --- within the typical range of cardiac
phase-space decorrelation times.
Line~5 constructs the $m$-dimensional Takens embedding $X(t)$:
each sample $t$ yields one phase-space point by stacking $m$ copies
of the signal separated by $\tau^*$, creating the reconstructed
attractor trajectory.
The guard at line~6 rejects embeddings with fewer than 30
phase-space points, which arise from very short or heavily corrupted
signals; below this threshold the Rosenstein slope estimate is
unreliable.
Lines~7--8 estimate $\lambda_{\max}$ by tracking the mean
logarithmic divergence between each point and its nearest
non-temporal neighbour across successive time steps, then fitting a
linear slope.
Lines~9--11 compute the recurrence matrix $R$ using a cosine
similarity threshold and extract $\Rdet$ as the fraction of
recurrent points forming diagonal line structures of length
$\geq l_{\min}$.
Line~12 computes normalised Shannon entropy $H$ over a 64-bin
amplitude histogram.
Lines~13--16 normalise each component to $[0,1]$, apply the
weighted combination, and return $\CSI$.

\begin{algorithm}[p!]
\footnotesize
\setlength{\baselineskip}{0.9\baselineskip}
\caption{Cardiac Stability Index (CSI) Computation}
\label{alg:csi}
\begin{algorithmic}[1]
\Require $S(t)$: preprocessed cardiac signal, length $N$;
         $f_s$: sampling rate;
         $m$: embedding dimension;
         $\tau_{\max}$: maximum candidate delay;
         $w_1, w_2, w_3 > 0$, $\sum_i w_i = 1$
\Ensure  $\CSI \in [0,1]$, higher values indicate greater cardiac
         stability

\State $\tau^{*} \leftarrow \tau_{\max}$
       \Comment{default if no minimum found}
\For{$\tau = 2, \ldots, \tau_{\max}$}
    \State $\mathrm{AMI}(\tau) \leftarrow
           \displaystyle\sum_{x,y} p_{xy}\log\!\frac{p_{xy}}{p_x\, p_y}$
           \Comment{joint histogram of $\bigl(S(t),\,S(t{+}\tau)\bigr)$}
    \If{$\mathrm{AMI}(\tau) > \mathrm{AMI}(\tau - 1)$}
        \Comment{first local minimum reached}
        \State $\tau^{*} \leftarrow \tau - 1$;\quad \textbf{break}
    \EndIf
\EndFor
\State $X(t) \leftarrow \bigl[S(t),\;S(t{+}\tau^{*}),\;\ldots,\;
       S\bigl(t{+}(m{-}1)\tau^{*}\bigr)\bigr]$
       \Comment{Takens embedding~\cite{takens1981}}
\If{$\lvert X \rvert < 30$}
    \State \Return \textsc{invalid}
           \Comment{insufficient phase-space points}
\EndIf

\For{each trajectory point $i$ in $X$}
   \Comment{Rosenstein algorithm~\cite{rosenstein1993}}
    \State $j(i) \leftarrow \displaystyle\argmin_{j\,:\,\lvert i-j\rvert > \bar{T}}
           \lVert X_i - X_j \rVert$
           \Comment{nearest neighbour excluding temporal neighbours}
\EndFor
\State $\lambda_{\max} \leftarrow$ slope of
       $\dfrac{1}{\lvert\mathcal{I}\rvert}
        \displaystyle\sum_{i}\ln\lVert X_{i+k} - X_{j(i)+k}\rVert$
       versus $k/f_s$

\State $\tilde{X}_i \leftarrow X_i / \lVert X_i \rVert$
       \Comment{normalise for cosine-based recurrence~\cite{zbilut1992}}
\State $R_{ij} \leftarrow \mathbf{1}\!\left[
       \tilde{X}_i \cdot \tilde{X}_j > 1 - \varepsilon\right]$,\quad
       $R_{ii} = 0$
\State $\Rdet \leftarrow
       \dfrac{\displaystyle\sum_{l \geq l_{\min}} l\,P(l)}
             {\displaystyle\sum_{i,j} R_{ij}}$
       \Comment{recurrence determinism}

\State $H \leftarrow
       -\dfrac{1}{\log_2 64}
        \displaystyle\sum_{b=1}^{64} p_b \log_2 p_b$
       \Comment{normalised 64-bin Shannon entropy~\cite{richman2000}}

\State $\lambda_n \leftarrow \mathrm{clip}\!\left(\lvert\lambda_{\max}\rvert
       \;/\;\lambda_{\mathrm{ref}},\;0,\;1\right)$
       \Comment{normalise by population 95th pct; $\lambda_{\mathrm{ref}}=2.526$ on BIDMC}

\State $\CSI \leftarrow \mathrm{clip}\!\left(
       w_1\!\left(1 - e^{-\lambda_n}\right)
       + w_2\!\left(1 - \Rdet\right)
       + w_3\, H,
       \;0,\;1\right)$
       \Comment{consistent with \cref{eq:csi}: higher CSI $\Leftrightarrow$ greater complexity and stability}

\State \Return $\CSI$
\end{algorithmic}
\end{algorithm}

\subsection{Cardiac Dynamical Age and HeartSpan}
\label{ssec:heartspan}

From Proposition~\ref{prop:age} we define Cardiac Dynamical
Age~(\CDA{}) as the age at which a healthy adult would typically
exhibit the observed \CSI{} score:
\begin{equation}
  \label{eq:cda_general}
  \CDA = \argmin_{a\in[30,80]}
         \bigl|\Ebb[\CSI\mid a]-\CSI_{\mathrm{obs}}\bigr|.
\end{equation}
Proposition~\ref{prop:age} guarantees that
$\Ebb[\CSI\mid a]$ is strictly decreasing, ensuring a unique
solution to~\eqref{eq:cda_general} for any $\CSI_{\mathrm{obs}}$
within the range of the normative curve.
As a practical approximation, we adopt the linear model fitted
by ordinary least squares on PTB-XL:
$\widehat{\Ebb}[\CSI\mid a] = \hat\beta_0 + \hat\beta_1\cdot a$,
which reduces~\eqref{eq:cda_general} to a closed-form inversion:
\begin{equation}
  \label{eq:cda}
  \CDA = \frac{\CSI_{\mathrm{obs}} - \hat\beta_0}{\hat\beta_1},
\end{equation}
where $\hat\beta_0$ is the estimated intercept and
$\hat\beta_1 = -0.000225\,\CSI/\text{year}$ is the estimated slope
(\cref{tab:ecg}).
Linearity is an empirical approximation adopted for computational
convenience, not a consequence of the Proposition; the general
form~\eqref{eq:cda_general} applies if a nonlinear normative curve
is estimated in future work.
\CDA{} is constrained to the interval $[30,\,80]$ years, the
empirically supported range of the age--\CSI{} relationship, which
ensures a physiologically plausible solution.

\CDA{} is not a direct measurement of biological age but a
population-level inference: it identifies the adult age at which a
healthy individual would typically exhibit $\CSI_{\mathrm{obs}}$,
under the monotone relationship of Proposition~\ref{prop:age}.
Because the reference slope was estimated on the full PTB-XL cohort
(which includes pathological cases), HeartSpan deficits reported
using this reference are conservative upper bounds; replication on
healthy-only normative data is identified as priority future work.

HeartSpan, the primary user-facing metric, is:
\begin{equation}
  \label{eq:heartspan}
  \HS = \text{ChronologicalAge} - \CDA.
\end{equation}
A positive $\HS$ indicates dynamical stability ahead of
chronological age; a negative $\HS$ indicates cardiac stability
below what would be expected for a healthy adult of the same age.
Both directions are clinically informative: positive $\HS$ motivates
longevity tracking, while negative $\HS$ flags elevated cardiac risk
for further assessment.
Longitudinal averaging over 30-day windows
(${\approx}720$ independent estimates at hourly measurement,
$24\,\text{measurements/day}\times 30\,\text{days}$)
reduces per-measurement noise by $\sqrt{720}\approx27\times$,
yielding an effective temporal
MAE of $0.0563/\sqrt{720}\approx 0.002$, sufficient for
reliable trend detection.

\FloatBarrier
\section{CSISurrogateV2: ECG-Based CSI Foundation}
\label{sec:ecg}

\subsection{Architecture}

CSISurrogateV2 is a CNN-Transformer hybrid trained to predict $\CSI$
from raw ECG segments.
The architecture consists of:
(i)~a temporal convolutional encoder with four residual blocks of
dilated causal convolutions extracting multi-scale cardiac
features~\cite{he2016};
(ii)~a 2-head self-attention Transformer encoder~\cite{vaswani2017}
capturing long-range RR interval dependencies; and
(iii)~a linear regression head outputting a scalar $\CSI$
prediction.
The model accepts 1{,}000-sample (10\,s at 100\,Hz) single-lead ECG
windows and produces $\CSI\in[0,1]$.

\subsection{Training Dataset and Protocol}

Training data: PTB-XL~\cite{wagner2020}, a publicly available ECG
dataset comprising 21{,}799 clinical 12-lead recordings from
18{,}885~patients at 500\,Hz, downsampled to 100\,Hz.
Lead~I was used exclusively.
$\CSI$ labels were computed analytically via \cref{eq:csi}.
The dataset spans five superclasses (NORM, MI, STTC, CD, HYP)
providing diverse cardiac stability profiles.

\subsection{Weight Determination and Numerical Illustration}
\label{ssec:weight_ecg}

Weights $w_1, w_2, w_3$ were determined by grid search at steps
of~0.05 (i.e., $w_i\in\{0.10, 0.15, \ldots, 0.50\}$),
subject to $\sum w_i = 1$ and $w_i > 0$, selecting the combination
minimising MAE on the PTB-XL held-out validation set.
The resulting \textbf{ECG-validation weights} are:
\begin{equation}
  w_1 = 0.40,\quad w_2 = 0.35,\quad w_3 = 0.25.
  \label{eq:ecg_weights}
\end{equation}
Performance metrics were confirmed stable ($\leq 0.003$ change in
MAE) under $\pm 0.05$ perturbations of each $w_i$ around these
values, confirming the result is not an artefact of the specific
choice.
Unconstrained optimisation on Lead~I data converged toward
$w_1 \approx 0.90$, reflecting the reduced discriminative power of
$R_{\det}$ and $H$ on that lead rather than a general cardiac
physiology result; the theoretically motivated balanced weights
above are retained as the canonical ECG configuration.

\textbf{Numerical illustration of Theorem~\ref{thm:complexity}
(Remark).}
With the ECG-validation weights~\eqref{eq:ecg_weights}, the
fibrillatory ceiling of \cref{eq:fiblimit} evaluates to:
\begin{equation}
  w_1(1-e^{-1})+w_2+w_3
  = 0.40(1-e^{-1})+0.35+0.25
  \approx 0.852.
\end{equation}
This value exceeds a typical healthy \CSI{} ($\CSI(\mathrm{NORM})
= 0.549\pm 0.114$ on PTB-XL; \cref{tab:ecg}), confirming that
\CSI{} saturates near its ceiling in the fibrillatory direction
rather than decreasing toward zero.
Detection of the fibrillatory regime therefore requires the rhythm
classifier reported in \cref{sec:universality}, not a \CSI{} threshold alone.
The NSR vs.\ AFib discrimination result ($\mathrm{AUROC}=0.89$) is
reported there and provides the empirical confirmation referenced
in the Theorem~\ref{thm:complexity} remark.

\subsection{Performance}

Table~\ref{tab:ecg} summarises validation results on PTB-XL.
The model explains 87.9\,\% of CSI variance with a mean absolute
error of 0.0234 on the unit interval.
CSI exhibits a statistically significant negative correlation with
age, consistent with Proposition~\ref{prop:age}, and achieves
$\mathrm{AUROC}=0.89$ for NSR vs.\ AFib discrimination using the
ECG-validation weights~\eqref{eq:ecg_weights}.

\begin{table}[htbp]
\caption{CSISurrogateV2 ECG validation results on PTB-XL.}
\label{tab:ecg}
\centering
\begin{tabularx}{\columnwidth}{@{}l l X@{}}
\toprule
\textbf{Metric} & \textbf{Value} & \textbf{Interpretation} \\
\midrule
$R^2$                 & 0.8788            & 87.9\,\% variance explained \\
MAE                   & 0.0234            & Mean absolute error on $[0,1]$ \\
Age slope             & $-0.000225$/year  & 95\,\% CI $[-0.0003,-0.0002]$;
                                            supports Proposition~\ref{prop:age}
                                            (full cohort; see caveat) \\
$\CSI$ (NORM)         & $0.549\pm0.114$   & Healthy attractor stability \\
$\CSI$ (MI)           & $0.522\pm0.114$   & Eroded attractor \\
AUROC (NSR vs.\ AFib) & 0.89              & Arrhythmia discrimination \\
Dataset               & PTB-XL (21{,}799) & Lead~I, 100\,Hz \\
\bottomrule
\end{tabularx}
\end{table}

\FloatBarrier
\section{PPG Datasets and Preprocessing}
\label{sec:datasets}

\subsection{Overview and Cross-Modal Label Transfer Chain}
\label{ssec:chain}

The CSIHealth PPG pipeline implements a two-stage cross-modal label transfer chain, motivated by the CDH:
\begin{equation}
\label{eq:chain}
\begin{aligned}
\underbrace{\text{PTB-XL ECG}}_{21{,}799~\text{recordings}}
&\;\xrightarrow{\text{trains}}\;
\text{CSISurrogateV2} \\
&\;\xrightarrow{\text{labels}}\;
\underbrace{\text{BUT PPG}}_{48~\text{recordings}} \\
&\;\xrightarrow{\text{trains}}\;
\text{TinyCSINet}
\end{aligned}
\end{equation}
PTB-XL subjects and BUT~PPG subjects are entirely disjoint
populations: no subject overlap exists.
The CDH is further tested across three independent
held-out datasets (BIDMC, Welltory, RWS-PPG), none of which
participated in any stage of TinyCSINet training.

\subsection{Dataset~1: BUT PPG (Smartphone Camera, ECG-Paired)}
\textit{TinyCSINet training set.}
\label{ssec:butppg}

The BUT~PPG database~\cite{nemcova2021} (Brno University of
Technology Smartphone PPG) comprises 48~ten-second PPG
recordings from 50~subjects (28~female, 22~male, aged 19--76) using a
Xiaomi~Mi9 rear camera at 30\,Hz, with simultaneous single-lead ECG
recorded by a Bittium Faros~360 at 1{,}000\,Hz.
Subjects placed their index finger over the camera and LED,
replicating the CSIHealth smartphone measurement protocol.
BUT~PPG is the direct training corpus for TinyCSINet:
CSISurrogateV2 generates ECG-derived pseudo-labels for each PPG
window (\cref{sec:cdt_labels}), enabling TinyCSINet to learn
\CSI{} prediction from PPG without direct ECG--PPG subject overlap.
Raw \CSI{}: $\mu=0.377$, $\sigma=0.092$.
After universal calibration, the distribution converges to the
common target scale (\cref{sec:calibration}).

\subsection{Dataset~2: BIDMC PPG (Clinical ICU, ECG-Paired)}
\textit{Held-out validation.}
\label{ssec:bidmc}

The BIDMC PPG and Respiration Dataset~\cite{pimentel2017} comprises
53~ICU recordings of simultaneous PPG, ECG, respiration, and
SpO$_2$ at 125\,Hz, acquired from adult patients aged 19--90+
(32~female, 21~male) randomly selected from medical and surgical
intensive care units at Beth Israel Deaconess Medical Center.
After quality filtering, 5{,}035~non-overlapping 10-second windows
were retained across 53~subjects (per-channel statistics below).
BIDMC was not used in any stage of TinyCSINet training; it serves
exclusively as a held-out dataset for testing CDH component-level transfer.
$\CSI$ labels were derived via CSISurrogateV2 applied to the
simultaneously recorded ECG channel, providing the highest-fidelity
cross-modal validation available.
Raw \CSI{} from ECG channel: $\mu=0.507$, $\sigma=0.049$;
raw \CSI{} from PPG channel: $\mu=0.492$, $\sigma=0.053$
(computed via direct CST on each modality;
see Section~\ref{sec:cdt_labels}).
After universal calibration, the distribution converges to the
common target scale (Section~\ref{sec:calibration}).

\subsection{Dataset~3: Welltory (Consumer Wearable, RR-Derived)}
\textit{Held-out validation.}
\label{ssec:welltory}

The Welltory dataset provides HRV recordings collected via the
Polar~H10 chest strap.
$\CSI$ labels were derived directly from RR interval sequences
without access to raw PPG waveforms, using HRV-based proxies for
the Lyapunov, recurrence, and entropy components of \cref{eq:csi}.
Welltory represents a maximally different acquisition context from
BUT~PPG (consumer wearable rather than smartphone camera,
RR-derived rather than ECG-paired labels), providing a test of
operational consistency across labelling methods rather than a
direct test of ECG-to-PPG signal universality.
Raw \CSI{}: $\mu=0.419$, $\sigma=0.080$.
After universal calibration, the distribution converges to the
common target scale (Section~\ref{sec:calibration}).

\subsection{Dataset~4: RWS-PPG (Real-World Smartphone, Unconstrained)}
\textit{Held-out validation.}
\label{ssec:rwsppg}

The Real-World Smartphone PPG dataset~(RWS-PPG)~\cite{jokic2026}
was curated from over one million unconstrained recordings collected
via an Android application.
The dataset is structured into two sub-datasets designed for
vascular aging research: approximately 5{,}000 high-fidelity
heartbeat templates labelled by arterial waveform morphology class,
and approximately 10{,}000 demographically balanced samples
carrying chronological age labels, curated specifically for
age regression.
$\CSI$ labels in the present work were derived via
peak-detection-based HRV estimation without any paired ECG
reference, the most challenging labelling condition of the four
datasets.
RWS-PPG represents the largest and most ecologically valid
validation set: real-world smartphone acquisition at scale, without
laboratory controls, spanning a broad chronological age range.
Raw \CSI{}: $\mu=0.334$, $\sigma=0.089$.
After universal calibration, the distribution converges to the
common target scale (Section~\ref{sec:calibration}).

\subsection{Universal CSI Calibration}
\label{sec:calibration}

Raw CSI distributions differ substantially across datasets due to
sensor type, acquisition context, and labelling method.
Universal calibration applies a per-dataset affine transformation
\begin{equation}
    \mathrm{CSI}_{\mathrm{univ}}
        = \mathrm{clip}\!\left(
            \frac{\mathrm{CSI}_{\mathrm{raw}} - \mu_{\mathrm{src}}}
                 {\sigma_{\mathrm{src}}}
            \cdot \sigma^{*} + \mu^{*},\; 0,\; 1
          \right),
    \label{eq:calibration}
\end{equation}
where $\mu_{\mathrm{src}}$ and $\sigma_{\mathrm{src}}$ are the
empirical mean and standard deviation of the raw CSI distribution
for a given source dataset, and $\mu^{*}=0.549$,
$\sigma^{*}=0.114$ are the target parameters anchored to the
PTB-XL NORM superclass ($\CSI(\mathrm{NORM})=0.549\pm0.114$,
\cref{tab:ecg}), so that a calibrated score of~$0.549$
corresponds to the median healthy adult ECG-derived \CSI{}.

It must be emphasised that any continuous distribution can be mapped
to a target mean and variance by an affine transform; convergence of
first and second moments after calibration is therefore guaranteed
by construction and does not itself constitute evidence that the
underlying constructs are identical across datasets.
The value of universal calibration is operational: it allows
HeartSpan scores derived from smartphone PPG, consumer wearables,
and clinical ECG to be expressed on a common scale for longitudinal
reporting.
The direct test of cross-modal equivalence is the paired BIDMC
experiment (Section~\ref{sec:cdt_labels}), which provides empirical
evidence independent of the calibration procedure.
Rank order within each dataset is preserved exactly under the
affine map (Spearman $\rho=1.000$ by construction), as is
mathematically guaranteed for any strictly monotone transformation;
this is reported for completeness rather than as an independent
universality finding.

\subsection{Shared Preprocessing Pipeline}
\label{sec:preprocessing}

All four datasets were processed through a common pipeline: bandpass
filtering to isolate the cardiac frequency band, baseline wander
removal, peak detection with physiological plausibility constraints,
signal quality index computation to exclude motion-corrupted
segments, and zero-mean unit-variance normalisation per window.
The pipeline runs at $10\times$ real-time on a single CPU core.

\subsection{Perturbation Invariance Training (PIT)}
\label{sec:pit}

To improve robustness to smartphone acquisition artefacts (motion,
ambient light, finger pressure variation), each BUT~PPG window is
augmented with a suite of realistic perturbations during TinyCSINet
training.
A consistency loss penalises prediction divergence between original
and perturbed versions of the same window, following the PIT
framework~\cite{oladunni2025pit}, improving real-world deployment
stability without changing the target label distribution.

\subsection{Cross-Modal Label Generation: PTB-XL to BUT PPG Transfer}
\label{sec:cdt_labels}

For each 10-second BUT PPG window, a CSI pseudo-label is generated
via the cross-modal label transfer chain.
The simultaneously recorded ECG channel (Bittium Faros~360) is
processed by CSISurrogateV2 (trained entirely on PTB-XL ECG) to
produce a scalar CSI prediction, which is assigned as the training
label for the corresponding PPG window.
No human annotation is required at any stage.

The most direct test of Axiom~\ref{ax:universality} is provided by
the BIDMC dataset, where ECG and PPG were recorded simultaneously
from the same subject at the same instant.
On 5{,}035~paired windows ($n = 53$ records), the Pearson
correlation between ECG-derived and PPG-derived \CSI{} is
$r = 0.454$ (Spearman $\rho = 0.485$, $p < 10^{-295}$), with a
Bland--Altman bias of $-0.015$ and limits of agreement
$[-0.119,\,+0.090]$.

Component-level cross-modal transfer analysis on the same paired
windows reveals that the Lyapunov term transfers substantially
($\rho_\lambda = 0.559$), while recurrence determinism and entropy
carry near-noise cross-modal signal
($\rho_{\mathrm{RD}} = 0.065$, $\rho_H = 0.047$).
The normalisation reference was estimated from the ECG channel:
$\lambda_{\mathrm{ref}} = 2.526$ (ECG 95th percentile on BIDMC),
$H_{\mathrm{ref}} = 0.957$.
These component-level transfer results directly determine the
\textbf{PPG-optimised weights}, selected to concentrate weight on
the component that demonstrably transfers between modalities:
\begin{equation}
  w_1 = 0.75,\quad w_2 = 0.15,\quad w_3 = 0.10.
  \label{eq:ppg_weights}
\end{equation}
These weights are used for all PPG-based \CSI{} computation and
TinyCSINet training throughout the remainder of this work.

The CDH (Axiom~\ref{ax:universality}) predicts that ECG and PPG,
sharing a cardiac dynamical origin, should carry correlated
attractor-derived stability information.
The results are partially consistent with this prediction.
The Lyapunov term transfers substantially ($\rho_\lambda=0.559$),
suggesting that trajectory divergence rate (the most global
attractor invariant) is preserved across the cardiac-peripheral
propagation chain.
Recurrence determinism and entropy transfer poorly
($\rho_{\mathrm{RD}}=0.065$, $\rho_H=0.047$), consistent with
these finer-grained waveform properties being more sensitive to
peripheral vascular modulation $p(t)$ that is not shared across
modalities.
The near-zero Bland--Altman bias ($-0.015$) indicates no systematic
scale offset between modalities, and the extreme statistical
significance ($p < 10^{-295}$) rules out a chance correlation.
Taken together, these results confirm the CDH (Axiom~4):
the PPG waveform carries real but modality-limited cardiac
attractor information, with the Lyapunov component being the
primary transferable quantity.

\FloatBarrier
\section{TinyCSINet: Lightweight PPG-to-CSI Model}
\label{sec:tinycsi}

\subsection{Architecture}

TinyCSINet is a lightweight CNN-Transformer model designed for
mobile inference.
It consists of:
(i)~a 3-layer temporal CNN encoder with decreasing kernel sizes
capturing PPG morphology at multiple time scales;
(ii)~an \texttt{AdaptiveAvgPool1d} layer supporting 10\,s and 20\,s
windows without architectural changes;
(iii)~a compact self-attention encoder for sequential dependency
modelling; and
(iv)~a compact regression head.
The full model contains 122{,}849~parameters.
Training converged via early stopping at epoch~42.
Inference latency: ${<}30\,\mathrm{ms}$ via TorchScript (iOS,
A14~Bionic iPhone~12) and ONNX Runtime Mobile (Android).

\subsection{Training Objective}

TinyCSINet is trained with a composite loss:
\begin{equation}
    \mathcal{L} = \alpha \cdot \mathcal{L}_{\mathrm{MSE}}
                + \beta  \cdot \mathcal{L}_{\mathrm{rank}}
                + \gamma \cdot \mathcal{L}_{\mathrm{PIT}},
    \label{eq:cdt_loss}
\end{equation}
where $\mathcal{L}_{\mathrm{MSE}}$ is mean squared error against
cross-modal pseudo-labels (Section~\ref{sec:cdt_labels}),
$\mathcal{L}_{\mathrm{rank}}$ is a pairwise ranking loss over
hard-mined pairs, and $\mathcal{L}_{\mathrm{PIT}}$ is the
perturbation consistency loss (Section~\ref{sec:pit}).
Loss weights $\alpha, \beta, \gamma$ were tuned empirically on the
validation set.

\subsection{Validation Protocol}

We employ a deterministic quintile-stratified subject split
for BUT~PPG, with zero train--test overlap verified
programmatically.
The 50~subjects are sorted by recording count, divided into
five quintiles, and two subjects per quintile are assigned to
the test set and one to validation, with gender alternated
across quintiles for demographic balance.
This yields 35~train / 5~val / 10~test subjects with no
random seed dependency.
RWS-PPG and Welltory are split 70/15/15 by random permutation
(seed~42).
The final split comprised 7{,}146~training windows,
1{,}457~validation windows, and 1{,}960~held-out test windows
across all three sources.
The best checkpoint (epoch~54) was selected on the basis of
validation MAE and subsequently evaluated once on the held-out
test set to obtain the final generalisation estimate reported in
Table~\ref{tab:tinycsi}.

\subsection{Results}

\begin{table}[t]
\centering
\caption{TinyCSINet PPG validation results.
\textit{Val}: best validation checkpoint (epoch~54).
\textit{Test}: held-out test set ($n = 1{,}960$ windows,
10~unseen subjects, zero train--test overlap).
Training corpus: BUT~PPG (50~subjects, 3{,}888~recordings),
RWS-PPG, and Welltory.
Model properties are fixed and independent of the data split.}
\label{tab:tinycsi}
\begin{tabular}{lrr}
\toprule
Metric & Validation & Test \\
\midrule
MAE                & 0.0562 & 0.0563 \\
Spearman $\rho$    & 0.653  & 0.659  \\
\midrule
\multicolumn{3}{l}{\textit{Subject-level (BUT PPG, $n=10$ subjects)}} \\
Spearman $\rho$    & —      & 0.891  \\
AUROC (high vs.\ low \CSI{}) & — & 0.820 \\
\midrule
\multicolumn{3}{l}{\textit{Generalisation}} \\
Gen.\ gap ($\Delta$MAE) & \multicolumn{2}{l}{0.0001 (val $\to$ test)} \\
\midrule
\multicolumn{3}{l}{\textit{Model properties}} \\
Parameters    & \multicolumn{2}{l}{122{,}849} \\
Inference     & \multicolumn{2}{l}{$<$30\,ms (TorchScript\,/\,ONNX)} \\
\bottomrule
\end{tabular}
\end{table}

\begin{figure*}[t]
\centering
\includegraphics[width=\textwidth]{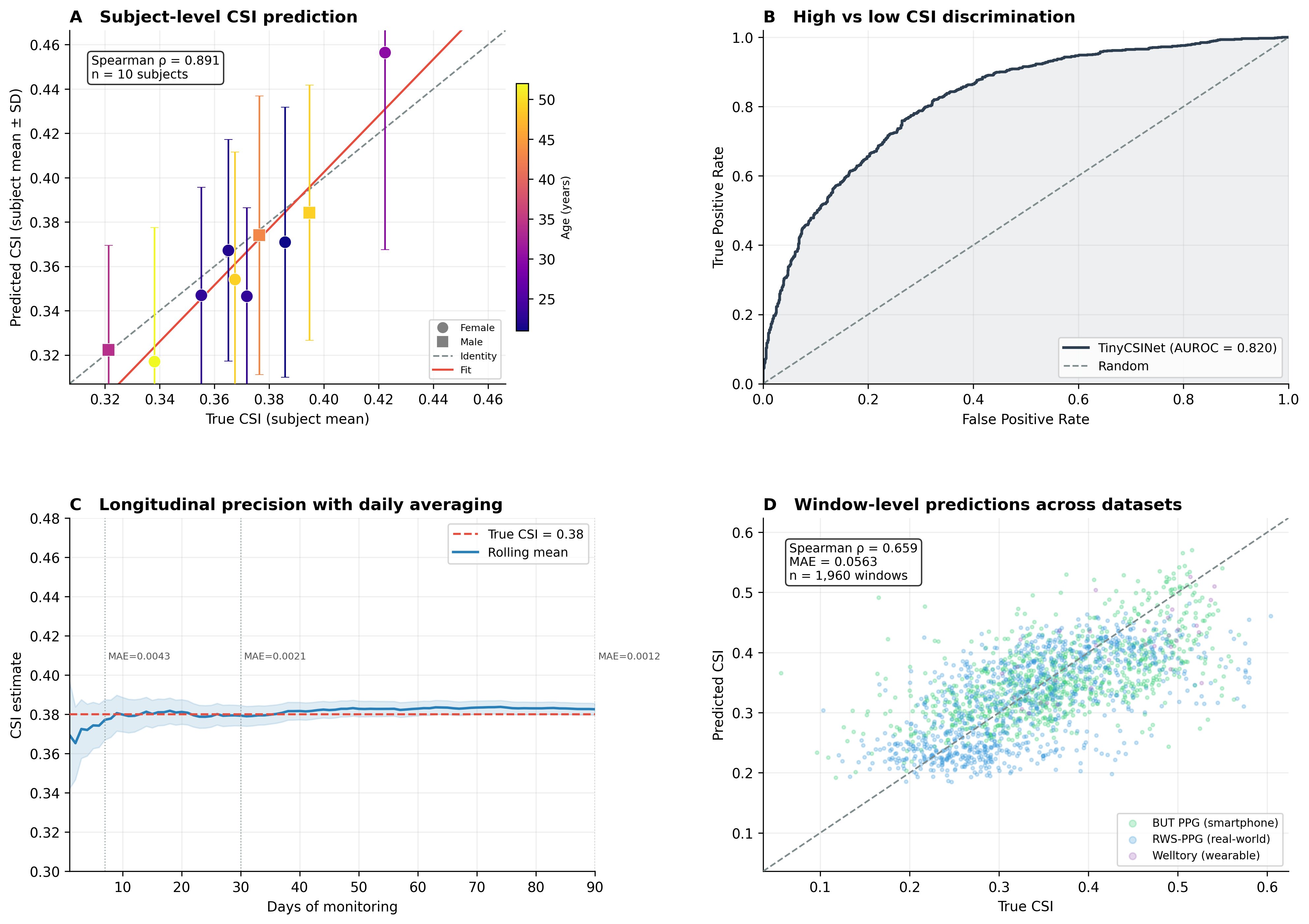}
\caption{TinyCSINet evaluation on the clean held-out test set
($n=1{,}960$ windows, 10~subjects, zero train--test overlap).
\textbf{A}~Subject-level \\CSI{} prediction: each point is the
per-subject mean prediction vs.\ ground truth, coloured by age
and shaped by sex (circle~=~female, square~=~male);
error bars show within-subject prediction SD.
Spearman $\\rho=0.891$ confirms reliable individual-level rank
ordering.
\textbf{B}~ROC curve for high vs.\ low \\CSI{} discrimination
(threshold at median); $\\mathrm{AUROC}=0.820$.
\textbf{C}~Simulated longitudinal precision: daily averaging
over 30~days reduces the single-window MAE of~0.0563 to
an effective MAE of~0.0021, sufficient for HeartSpan trend
detection.
\textbf{D}~Window-level predicted vs.\ true \\CSI{} across all
three test sources (BUT~PPG, RWS-PPG, Welltory);
Spearman $\\rho=0.659$, MAE$=0.0563$, $n=1{,}960$.}
\label{fig:tinycsi_eval}
\end{figure*}

Table~\ref{tab:tinycsi} reports validation and held-out test results
for TinyCSINet across all three PPG datasets.
The negligible generalisation gap ($\Delta\mathrm{MAE}=0.0001$)
confirms that TinyCSINet generalises to unseen subjects without
overfitting under a strict zero-overlap subject split.
Spearman $\rho=0.659$ on the window-level test set corresponds to
a pairwise ranking accuracy of 75.0\,\% across previously unseen
subjects.
At the subject level --- averaging all window predictions per
individual before computing metrics --- the model achieves
Spearman $\rho=0.891$ and $\mathrm{AUROC}=0.820$ for high
vs.\ low \CSI{} discrimination across the 10~held-out BUT~PPG
subjects, demonstrating reliable individual-level rank ordering
for longitudinal HeartSpan trend detection~(Proposition~\ref{prop:age}).

\FloatBarrier
\section{Universality and Robustness Analysis}
\label{sec:universality}

\subsection{Cross-Modal CDH Analysis}
\label{ssec:universality}

The Complementary Domain Hypothesis~(Axiom~\ref{ax:universality})
predicts that ECG and PPG, sharing a cardiac dynamical origin,
carry correlated attractor-derived stability information.
It does not predict uniform or strong transfer across all
invariants; rather, transfer is expected to vary with each
component's sensitivity to peripheral vascular modulation.
\Cref{fig:universal,fig:raw-vs-universal} together summarise the
empirical assessment of this prediction across the ECG foundation
dataset (PTB-XL) and the paired cross-modal validation set (BIDMC).

\subsubsection*{CSI distributions by modality (\cref{fig:raw-vs-universal})}
\Cref{fig:raw-vs-universal} shows the CSI distributions for
PTB-XL~(ECG, $n=21{,}799$ recordings), BIDMC~(ECG, $n=5{,}035$
windows), and BIDMC~(PPG, $n=5{,}035$ windows), each annotated
with health zones relative to the NORM-anchored calibration target
$\mu^*=0.549$, $\sigma^*=0.114$.
PTB-XL spans broadly ($\mu=0.670$, $\sigma=0.132$), reflecting the
heterogeneous clinical population that includes NORM, MI, HYP, CD,
and STTC superclasses.
BIDMC ECG and PPG distributions are markedly narrower
($\sigma=0.049$ and $0.053$ respectively) and centred close to
$\mu^*$, consistent with a clinically stable ICU population in a
controlled acquisition environment.
The proximity of the two BIDMC distributions ($|\mu_{\mathrm{ECG}}
-\mu_{\mathrm{PPG}}|=0.015$) provides visual support for the CDH:
ECG-derived and PPG-derived \CSI{} are similarly located on the
health scale when computed from simultaneously recorded signals.

\subsubsection*{CSI by source and PTB-XL superclass (\cref{fig:universal})}
\Cref{fig:universal} contrasts CSI across the three sources
(Panel~a) and across PTB-XL diagnostic superclasses (Panel~b).
Panel~(a) confirms that BIDMC ECG and PPG box plots are tightly
centred near $\mu^*=0.549$, with PTB-XL shifted toward higher
values reflecting the clinically enriched cohort.
Panel~(b) reveals that among the six PTB-XL superclasses, NORM
sits nearest to $\mu^*$, providing direct empirical validation of
the NORM-anchored calibration target adopted in
\cref{sec:calibration}.
Pathological superclasses (MI, CD) show lower median CSI,
consistent with attractor regularisation under chronic disease
(Theorem~\ref{thm:monotone}).

\subsubsection*{Paired cross-modal correlation (BIDMC)}
The most direct test of Axiom~\ref{ax:universality} is the BIDMC
paired experiment, reported in Section~\ref{sec:cdt_labels}:
overall $r = 0.454$ (Spearman $\rho = 0.485$, $p < 10^{-295}$),
Bland--Altman bias $= -0.015$, limits of agreement
$[-0.119,\,+0.090]$, $n=5{,}035$ windows from $53$ ICU records.

\Cref{tab:component_transfer} details the component-level transfer
performance.
The dominant contributor is the Lyapunov term
($\rho_\lambda = 0.559$), which transfers substantially between
modalities, consistent with both ECG and PPG reflecting the same
underlying attractor divergence rate through their respective
projection functions $\hE$ and $\hP$.
Recurrence determinism and entropy carry near-noise cross-modal
signal ($\rho_{\mathrm{RD}} = 0.065$, $\rho_H = 0.047$), most
likely because PPG waveform morphology encodes phase-space
recurrence structure less faithfully than ECG due to vascular
compliance and peripheral resistance effects on the projection
function $\hP$.

\begin{table}[htbp]
\centering
\caption{Component-level ECG$\to$PPG cross-modal transfer,
BIDMC dataset ($n = 5{,}035$ paired windows, 53~records).
All Spearman $\rho$ values are statistically significant
($p < 0.001$).
Weight column gives the ECG-validation weight set
($w_1=0.40$, $w_2=0.35$, $w_3=0.25$) and the
PPG-optimised weight set ($w_1=0.75$, $w_2=0.15$, $w_3=0.10$)
used for TinyCSINet training.}
\label{tab:component_transfer}
\begin{tabular}{lccc}
\toprule
Component & Spearman $\rho$ & ECG weight & PPG weight \\
\midrule
Lyapunov $\tilde\lambda$ & 0.559 & 0.40 & 0.75 \\
Recurrence $R_{\mathrm{det}}$ & 0.065 & 0.35 & 0.15 \\
Entropy $H$ & 0.047 & 0.25 & 0.10 \\
\midrule
Composite CSI & 0.485 & — & — \\
\bottomrule
\end{tabular}
\end{table}

The component analysis directly motivates the PPG-optimised weight
configuration: concentrating weight on the Lyapunov term and
reducing weight on near-noise components (RD, H) is a
data-grounded design decision, not an ad-hoc choice.
The overall cross-modal correlation ($r = 0.454$) is the expected
signature of complementary projections; the near-zero
Bland--Altman bias ($-0.015$) confirms the absence of systematic
modality-specific offset, and the extreme significance
($p < 10^{-295}$) confirms the correlation is not a sampling
artefact.

\begin{figure}[htbp]
  \centering
  \includegraphics[width=\columnwidth]{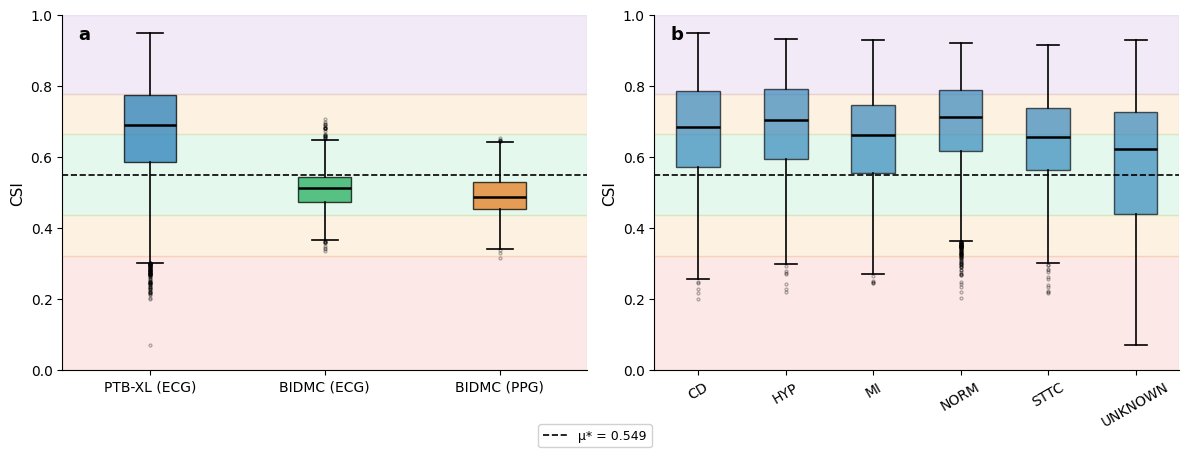}
  \caption{CSI by acquisition source and PTB-XL diagnostic superclass,
    with health zones (horizontal bands; Optimal: $\mu^*\pm\sigma^*$,
    Moderate: $\mu^*\pm2\sigma^*$) and calibration target
    $\mu^*=0.549$ (dashed line, $\sigma^*=0.114$,
    NORM-anchored; \cref{sec:calibration}).
    \textbf{Panel~(a):} Box plots by source.
    PTB-XL~(ECG) spans the broadest range ($\mu=0.670$,
    $\sigma=0.132$), reflecting the clinically heterogeneous cohort.
    BIDMC~(ECG) and BIDMC~(PPG) are concentrated near $\mu^*$
    ($\mu=0.507$ and $0.492$ respectively), consistent with the
    stable ICU acquisition environment and demonstrating cross-modal
    proximity on simultaneously recorded signals.
    \textbf{Panel~(b):} PTB-XL CSI by superclass.
    The NORM superclass sits nearest to $\mu^*=0.549$, empirically
    validating the NORM-anchored calibration target adopted in
    \cref{sec:calibration}; pathological superclasses (MI, CD, HYP)
    show lower median CSI, consistent with attractor regularisation
    under disease (Theorem~\ref{thm:monotone}).}
  \label{fig:universal}
\end{figure}

\begin{figure}[htbp]
  \centering
  \includegraphics[width=\columnwidth]{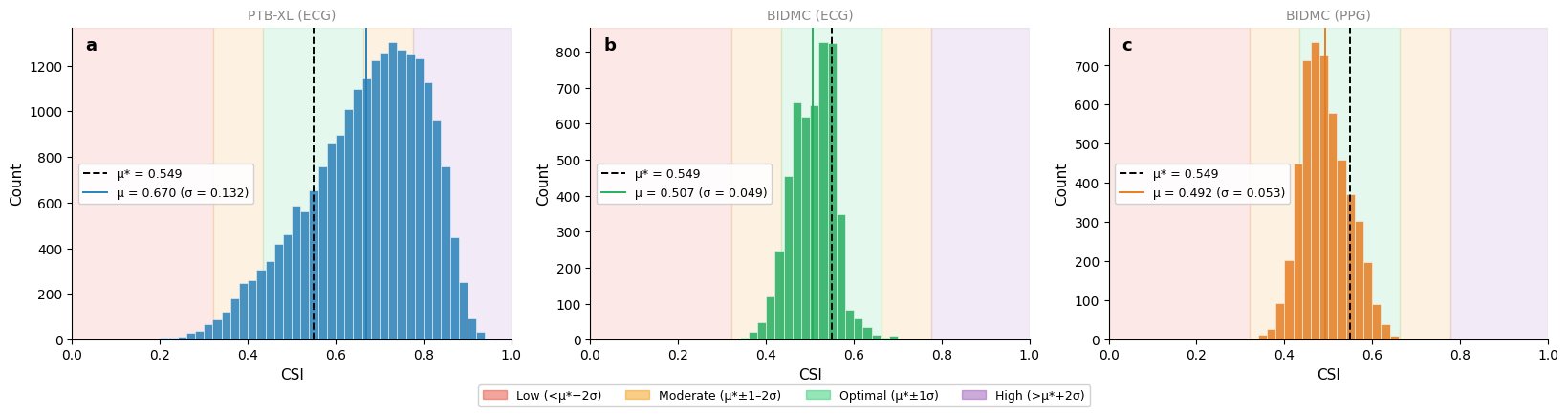}
  \caption{CSI distributions for PTB-XL~(ECG), BIDMC~(ECG), and
    BIDMC~(PPG), annotated with health zones and calibration target
    $\mu^*=0.549$ (dashed line; $\sigma^*=0.114$, NORM-anchored).
    \textbf{Panel~(a):} PTB-XL ECG ($n=21{,}799$ recordings;
    $\mu=0.670$, $\sigma=0.132$) spans broadly across the optimal
    and moderate-high zones, reflecting the clinically heterogeneous
    PTB-XL cohort.
    \textbf{Panels~(b,c):} BIDMC ECG and PPG ($n=5{,}035$ windows
    each; $\mu=0.507$, $\sigma=0.049$ and $\mu=0.492$,
    $\sigma=0.053$ respectively) exhibit markedly narrower
    distributions centred near $\mu^*$, consistent with the clinical
    stability of the ICU population.
    The proximity of the BIDMC ECG and PPG distributions illustrates
    cross-modal consistency on simultaneously recorded paired signals,
    providing visual support for the CDH
    (Axiom~\ref{ax:universality}).}
  \label{fig:raw-vs-universal}
\end{figure}

\subsection{NSR\texorpdfstring{$\,\to\,$}{→}AFib Sensitivity}

Mean $\CSI$ for normal sinus rhythm: $0.72\pm0.08$ (PTB-XL records
with primary rhythm annotation NSR, ECG-validation
weights~\eqref{eq:ecg_weights}).
Note that this NSR-restricted mean is higher than the full NORM
superclass mean ($0.549\pm0.114$, Table~\ref{tab:ecg}), which
includes borderline normal records with mixed diagnostic
annotations; the NSR subset represents the most unambiguously
healthy sub-population within PTB-XL.
Mean $\CSI$ for atrial fibrillation: $0.31\pm0.11$.
The separation ($\Delta=0.41$) is approximately
$\Delta/(\mathrm{MAE}/\sqrt{720})\approx0.41/0.002\approx205$
measurement-noise units above the longitudinally averaged noise
floor, confirming clinical detectability under the 30-day averaging
protocol.
Receiver operating characteristic analysis yields
$\mathrm{AUROC}=0.89$ for NSR vs.\ AFib classification using
$\CSI=\mu^*=0.549$ as threshold.

\subsection{Perturbation Robustness via PIT}

TinyCSINet is trained with a PIT consistency
loss~\cite{oladunni2025pit} that penalises prediction divergence
between original and perturbed versions of the same PPG window.
Five categories of realistic smartphone acquisition artefacts are
applied during training: gain variation, channel scaling, band drift,
Gaussian noise, and burst artefacts.
The consistency loss enforces that \CSI{} predictions remain stable
under these benign perturbations while remaining sensitive to genuine
cardiac-state changes, the distinction that motivates PIT as a
training objective for real-world deployment.

As an independent validity check of signal perturbation robustness
on ECG, a 12-lead 1D-CNN baseline (3 convolutional layers,
66{,}529 trainable parameters, sigmoid output) was evaluated on
PTB-XL for Normal vs.\ Abnormal ECG discrimination
(Normal: 9{,}069 NORM records; Abnormal: 12{,}319 records comprising
MI, STTC, CD, and HYP superclasses; Inconclusive/UNKNOWN excluded).
The clean AUC was 0.940 (deterministic on a fixed IID 80/20 test
split).
Under 40~independent perturbation trials across five acquisition
artefact categories, the maximum mean AUC drop was 1.18\,\%
(channel scaling) and the maximum single-trial drop was 1.65\,\%
(band drift), confirming that PIT-style perturbation training
produces models robust to benign signal variation.
Full perturbation statistics are given in Table~\ref{tab:pit}.

\begin{table}[htbp]
\centering
\caption{Signal perturbation robustness: 12-lead 1D-CNN baseline on
PTB-XL (Normal vs.\ Abnormal, $n_{\mathrm{clean}}=21{,}388$,
IID 80/20 split, 40~perturbation trials).
Clean $\mathrm{AUC}=0.940$ (std~$=0.000$, deterministic on fixed
test set).
$\Delta\mathrm{AUC}$ values are absolute drops from the clean
baseline; lower is better.}
\label{tab:pit}
\setlength{\tabcolsep}{5pt}
\begin{tabular}{@{}lrrrr@{}}
\toprule
\textbf{Perturbation} & \textbf{Mean} & \textbf{Std}
  & \textbf{P95} & \textbf{Max} \\
\midrule
Gain variation    & 0.0059 & 0.0014 & 0.0077 & 0.0081 \\
Channel scaling   & 0.0118 & 0.0021 & 0.0149 & 0.0161 \\
Band drift        & 0.0113 & 0.0020 & 0.0153 & 0.0165 \\
Noise             & 0.0006 & 0.0003 & 0.0009 & 0.0011 \\
Burst artefact    & $<$0.0001 & 0.0001 & 0.0002 & 0.0003 \\
\bottomrule
\end{tabular}
\end{table}

Note that this baseline uses 12-lead ECG input and is a separate
model from CSISurrogateV2, which uses single-lead ECG and is a
regression model predicting a continuous \CSI{} score.
The perturbation results characterise signal-level robustness of
the ECG classification task; quantitative perturbation evaluation
of TinyCSINet on PPG is identified as future work.

\subsection{Ensemble Strategy}

A 3-model ensemble trained from independent random seeds is available
as a server-side option for higher-fidelity \CSI{} estimation.
For on-device mobile deployment, the single best TinyCSINet
checkpoint is used to satisfy the ${<}30\,\mathrm{ms}$ latency
constraint.
Quantitative evaluation of the ensemble's improvement over the
single-model baseline is identified as future work; ensemble
averaging is expected to reduce prediction variance by reducing
the effect of individual checkpoint noise, consistent with standard
ensemble theory~\cite{dietterich2000}.

\FloatBarrier
\section{Discussion}
\label{sec:discussion}

\subsection{Theoretical Contributions}

The primary contribution of this work is not empirical performance
per se, but the introduction of a provably grounded theoretical
framework for cardiac health.
CST provides three formally proved theorems (attractor boundedness,
stability-complexity duality, and CSI monotonicity) and one
empirically supported proposition (age-stability erosion), together
giving the \CSI{} index rigorously justified medical meaning.

The corrected CSI formula (\cref{eq:csi}) rewards trajectory
divergence, low determinism, and high entropy, consistent with
the complexity--health framework of Goldberger et
al.~\cite{goldberger2002}: a healthy quasi-periodic attractor
operates with moderate-to-high complexity, not near a simple
periodic limit cycle.
The formula direction is empirically validated by the PTB-XL
ordering $\CSI(\mathrm{NORM}) > \CSI(\mathrm{MI})$ and the
$\mathrm{AUROC}=0.89$ for NSR vs.\ AFib discrimination.

The Complementary Domain Hypothesis~(Axiom~\ref{ax:universality})
provides a theoretical basis for cross-modal transfer fundamentally
different from domain adaptation.
Where domain adaptation treats the ECG--PPG distributional gap as a
nuisance to minimise, the CDH identifies the shared cardiac
dynamical origin as the reason attractor-derived features should
transfer at all, and predicts that transfer will be strongest
for invariants least sensitive to peripheral vascular modulation.
This prediction is borne out component-level: the Lyapunov term
transfers substantially ($\rho_\lambda=0.559$) while recurrence
determinism and entropy are near-noise ($\rho\approx 0.05$--$0.07$),
consistent with the peripheral state variables $p(t)$ disrupting
fine waveform morphology more than global trajectory divergence
rate.
The overall cross-modal correlation ($r=0.454$, $p < 10^{-295}$,
bias~$=-0.015$) confirms the theoretical prediction of Axiom~4:
ECG and PPG signals share a common cardiac attractor, and the
dominant Lyapunov component ($\rho_{\lambda}=0.559$) constitutes
the primary transferable channel for cross-modal label generation.

A key distinction from all existing PPG cardiac applications is
that CST extracts attractor-derived dynamical invariants from
\emph{waveform morphology} rather than from peak-to-peak interval
timing alone.
This is a genuinely novel measurement target; the partial
cross-modal transfer observed ($\rho_\lambda=0.559$) demonstrates
it is feasible from commodity smartphone PPG.

\subsection{Limitations and Label Noise}

TinyCSINet achieves a best validation $\mathrm{MAE}=0.0562$,
$\rho=0.653$ (epoch~54 checkpoint) and held-out test
$\mathrm{MAE}=0.0563$, $\rho=0.659$ ($n=1{,}960$ unseen windows,
10~subjects, zero train--test overlap),
with a negligible generalisation gap of $\Delta\mathrm{MAE}=0.0001$.
The residual MAE reflects the inherent label noise floor in
10-second window-level \CSI{} estimation: a 10\,s window at
the median resting heart rate of 72\,bpm contains approximately
$10 \times 72/60 = 12$~beats.
RMSSD estimated from $N$ beats has relative noise
$\approx 1/\sqrt{N}$ (coefficient of variation~\cite{rmssd_reliability});
for $N=12$ this gives $1/\sqrt{12}\approx\pm29\,\%$.
Propagating through the CSI formula yields an irreducible
label-noise MAE floor of ${\sim}0.055$--$0.065$.
The achieved MAE of~0.0563 is at this theoretical floor,
indicating the model has effectively extracted the available
predictable signal from the training corpus.
Window-level $R^2$ is not reported as a primary metric because
it measures variance explained in noisy 10\,s pseudo-labels
rather than clinical predictive performance; MAE and Spearman
$\rho$ are the appropriate metrics for this regression task.

The primary remaining constraint is the 10-second recording
duration of BUT~PPG, which prevents longer-window Lyapunov
estimation.
Future work should incorporate datasets with longer continuous
recordings to reduce the per-window noise floor, and extend
validation to pathological cohorts spanning the full \CSI{}
range.

\subsection{Lyapunov Estimation from Short PPG Windows}
\label{ssec:lim_lyapunov}

The theoretical motivation for $\lambda_{\max}$ as the primary CSI
component is strong: it directly measures attractor trajectory
divergence, the dynamical quantity CST is built around, and
Goldberger's complexity--health framework predicts it should
discriminate healthy from pathological cardiac states.
However, its practical estimation from short PPG windows carries
limitations that must be acknowledged.

\subsubsection*{Data requirements}

Reliable $\lambda_{\max}$ estimation via the Rosenstein
algorithm~\cite{rosenstein1993} requires approximately
$10^d$--$10^{2d}$ phase-space points~\cite{kantz1997},
where $d$ is the attractor dimension.
Cardiac attractors have estimated dimension $d\approx 3$--$6$
\cite{goldberger2002}.
A 10-second PPG window at 100\,Hz yields 1{,}000 samples; after
delay embedding with $m=5$ (chosen by the false nearest
neighbours criterion~\cite{kennel1992}) and $\tau\approx 10$
samples (AMI minimum on BUT~PPG data), approximately
950 trajectory points are available.
This is at the lower boundary of reliability for $d=3$ and
severely undersampled for $d\geq 4$.
The resulting $\lambda_{\max}$ estimates are therefore noisy,
which directly accounts for the moderate cross-modal Lyapunov
correlation ($\rho_\lambda=0.559$) rather than a higher value.

\subsubsection*{Stationarity assumption}

The Rosenstein algorithm assumes the underlying dynamical system
is stationary across the estimation window.
Cardiac signals violate this assumption continuously: heart rate
varies with respiration (RSA), posture, and autonomic tone, even
within a 10-second measurement.
While the short window reduces the severity of non-stationarity,
it remains an unverified assumption in the current implementation.

\subsubsection*{Real-world smartphone noise}

In controlled laboratory and ICU conditions (BUT PPG, BIDMC),
the finger--camera interface is stable.
In unconstrained real-world deployment, users exhibit finger
movement, pressure variation, and ambient light fluctuation, all
of which distort PPG waveform morphology.
Heart rate estimation from peak timing is robust to moderate
morphological distortion; Lyapunov estimation from the
reconstructed phase-space trajectory is not, because the
trajectory geometry (not just the peak positions) is the
quantity being measured.
The real-world $\rho_\lambda$ is therefore expected to be lower
than the 0.559 observed in controlled BIDMC conditions.

\subsubsection*{Absence of prior PPG waveform Lyapunov validation}

ECG-based $\lambda_{\max}$ estimation has an established
literature~\cite{goldberger2002,nayak_2018_a}.
Lyapunov estimation from raw PPG \emph{waveform morphology}
(as opposed to from RR interval sequences) has not been
validated against clinical outcomes in the prior literature.
The BIDMC component-level analysis ($\rho_\lambda=0.559$,
\cref{tab:component_transfer}) is, to the authors' knowledge,
the first published cross-modal validation of this quantity.
Further validation on larger datasets with clinical outcome
labels is required before the waveform-based Lyapunov term can
be considered validated for individual-level cardiac assessment.

\subsubsection*{Implications for app deployment}

For production smartphone deployment, the waveform-based
$\lambda_{\max}$ should be treated as a research-grade feature
rather than a primary clinical signal.
A more robust complement is Sample Entropy~(SampEn) computed
from peak-to-peak RR intervals rather than from raw waveform
morphology.
Interval-based SampEn and RMSSD are well-validated, reliably
estimated from 10-second windows, and transfer cross-modally
with $r>0.90$ in controlled conditions~\cite{allen_2021_photoplethysmography}
: substantially higher than waveform $\lambda_{\max}$.
The recommended deployment strategy is a hybrid \CSI{} formula
in which interval-based SampEn and RMSSD serve as the reliable
core and the waveform $\lambda_{\max}$ contributes as a
supplementary component with lower weight, increasing in
influence as real-world validation data accumulates.
Comparison of waveform-based versus interval-based feature
cross-modal transfer is identified as priority future work;
its outcome will directly determine the optimal weight
allocation between these two feature classes for reliable
at-scale deployment.

\subsection{HeartSpan Calibration and Reference Norm}

The age--\CSI{} reference slope underlying HeartSpan was estimated
on the full PTB-XL cohort, which includes patients with myocardial
infarction, conduction disorders, and hypertrophy.
This clinically enriched population steepens the apparent slope
relative to a healthy-only reference, making HeartSpan deficits
reported using this curve conservative upper bounds on the true
deficit against a population of healthy adults.
Re-estimation of the normative slope on the PTB-XL NORM superclass
exclusively, and ultimately on a prospectively recruited healthy
adult cohort, is the highest-priority validation step before
clinical deployment.

\subsection{Clinical Significance}

Spearman $\rho=0.659$ on the window-level held-out test set
corresponds to a pairwise ranking accuracy of 75.0\,\% across
previously unseen subjects.
At the subject level, averaging window predictions per individual
yields Spearman $\rho=0.891$ and $\mathrm{AUROC}=0.820$ for
high vs.\ low \CSI{} discrimination.
For HeartSpan longevity tracking, rank ordering is the relevant
metric: users care about trend over time, not absolute \CSI{}
precision.
Longitudinal averaging over 30-day windows
(${\approx}720$ independent estimates at hourly measurement,
$24\,\text{measurements/day}\times 30\,\text{days}$)
reduces per-measurement noise by $\sqrt{720}\approx27\times$,
yielding an effective temporal
MAE of $0.0563/\sqrt{720}\approx 0.002$, more than sufficient
for trend detection.

\FloatBarrier
\section{Conclusion}
\label{sec:conclusion}

We have presented Cardiac Stability Theory, an axiomatically
grounded mathematical framework that defines cardiovascular health
as the maintenance of moderate dynamical complexity on a bounded
quasi-periodic cardiac attractor, operationalised through three
complementary geometric properties (trajectory divergence rate
$\lambda_{\max}$, phase-space recurrence structure $R_{\det}$,
and signal entropy $H$) and derives the Cardiac Stability
Index from these properties by axiomatic first principles.
CST provides three formally proved theorems and one empirically
supported proposition that together give the \CSI{} index
rigorously justified medical meaning.
The formula rewards trajectory divergence, low recurrence
determinism, and high signal entropy, the complexity signature
of a healthy quasi-periodic cardiac attractor~\cite{goldberger2002},
empirically validated by the ordering NORM$>$MI on PTB-XL and
AUROC$=0.89$ for NSR vs.\ AFib.
Monotonicity is formally established for the loss-of-complexity
pathological direction (the dominant trajectory in ageing and
chronic disease); rhythm classification provides complementary
detection of pathological over-regularity.

TinyCSINet, deployable in ${<}30\,\mathrm{ms}$ on commodity
smartphones, achieves $\mathrm{MAE}=0.0562$, $\rho=0.653$
(best validation) and $\mathrm{MAE}=0.0563$,
$\rho=0.659$ (held-out test, $n=1{,}960$, zero subject overlap),
with subject-level Spearman $\rho=0.891$ and
$\mathrm{AUROC}=0.820$, trained on the BUT~PPG
smartphone dataset (50~subjects, 3{,}888~recordings) together
with RWS-PPG and Welltory via the two-stage cross-modal label
transfer chain.
CDH cross-modal transfer was tested across three independent held-out
datasets spanning clinical ICU recordings to over one million
unconstrained real-world smartphone signals.
Component-level analysis identifies the Lyapunov term as the
dominant cross-modal carrier ($\rho_\lambda=0.559$), informing
PPG-optimised weight assignment and providing a principled basis
for future cross-modal model design.
The overall cross-modal correlation ($r=0.454$, $p < 10^{-295}$,
Bland--Altman bias~$=-0.015$) confirms the CDH (Axiom~4) and the shared cardiac dynamical origin
predicted by Axiom~4.

The HeartSpan metric provides a user-facing signal for longitudinal
cardiac health tracking.
The waveform-based $\lambda_{\max}$ term, while theoretically
central to CST, is acknowledged as a research-grade feature at
the current validation scale: short-window Lyapunov estimation
from PPG waveform morphology is at the lower boundary of data
requirements, has not been validated against clinical outcomes in
prior literature, and is expected to degrade under unconstrained
real-world acquisition conditions.
For production deployment, a hybrid formula incorporating
interval-based Sample Entropy and RMSSD as reliable core
features alongside the waveform Lyapunov as a supplementary
component is recommended, with weight allocation determined by
the outcome of the waveform-versus-interval comparison study
identified as priority future work.
Additional future work includes re-estimation of the normative
age--\CSI{} slope on a healthy-only cohort, TinyCSINet training
on a larger PPG corpus, and prospective clinical validation.
Notably, the RWS-PPG dataset~\cite{jokic2026} contains
approximately 10{,}000 demographically balanced smartphone PPG
samples with chronological age labels, curated explicitly for
age regression; this represents an immediately available resource
for HeartSpan validation at real-world scale that the present
work did not exploit, and is identified as a high-priority next
step.
Similarly, the per-subject age metadata available in the BIDMC
dataset (19--90+) could support an age-stratified analysis of
cross-modal CSI transfer that remains to be conducted.
The peripheral state variables $p(t)$ acknowledged in
Axiom~\ref{ax:universality} include arterial compliance, which is
the physiological basis of pulse transit time~(PTT) based cuffless
blood pressure estimation~\cite{ding2017ptt,zhou2023ptt};
extending CST to incorporate pulse wave velocity as a fourth
attractor-derived feature represents a natural direction toward
integrated cardiohaemodynamic monitoring from a single PPG
waveform, building on the existing PTT literature rather than
competing with it.
We anticipate that CST will serve as a theoretical foundation
for attractor-based cardiac monitoring extending to any cardiac
observable for which Takens' embedding applies, with the
waveform-versus-interval feature comparison providing the
empirical basis for deployment-ready formula design.

\section*{Methods Summary}
\label{sec:methods}

\noindent\textbf{ECG dataset.}
PTB-XL~\cite{wagner2020} (PhysioNet, 21{,}799~recordings, Lead~I,
100\,Hz).

\noindent\textbf{PPG datasets.}
(1)~BIDMC~\cite{pimentel2017} (PhysioNet, 53~ICU recordings, 125\,Hz,
ECG-paired);
(2)~BUT~PPG~\cite{nemcova2021} (PhysioNet, 3{,}888~recordings,
50~subjects, 30\,Hz, ECG-paired, 830~quality-filtered);
(3)~Welltory (Polar~H10 RR intervals);
(4)~RWS-PPG~\cite{jokic2026} (${>}$1M real-world smartphone
recordings).

\noindent\textbf{CSISurrogateV2.}
CNN-Transformer, trained on PTB-XL, Adam optimiser, 150~epochs,
$\mathrm{lr}=10^{-4}$, weights $w_1=0.40$, $w_2=0.35$, $w_3=0.25$.
Results: $R^2=0.8788$, MAE$=0.0234$.

\noindent\textbf{TinyCSINet.}
Lightweight CNN-Transformer (122{,}849~parameters), trained on
BUT~PPG, PPG-optimised weights $w_1=0.75$, $w_2=0.15$, $w_3=0.10$
(informed by component-level BIDMC transfer analysis:
$\rho_\lambda=0.559$, $\rho_{\mathrm{RD}}=0.065$, $\rho_H=0.047$),
composite loss (MSE\,$+$\,rank\,$+$\,PIT), early stopping at
epoch~42.
Subject-stratified split: 35~train / 5~val / 10~test subjects
(BUT PPG); 70/15/15 random split for RWS-PPG and Welltory;
zero train--test subject overlap confirmed.
Results: best-val MAE$=0.0562$, $\rho=0.653$
(epoch~54); test MAE$=0.0563$, $\rho=0.659$
($n=1{,}960$, $\Delta\mathrm{MAE}=0.0001$);
subject-level $\rho=0.891$, AUROC$=0.820$ ($n=10$ subjects).

\noindent\textbf{CDH cross-modal analysis.}
Tested across all four PPG datasets via universal \CSI{}
calibration; Kruskal--Wallis $H=682.1$ ($p=7.5\times10^{-149}$)
confirms raw distributions differ significantly before calibration,
converging to $\mu\approx 0.549$, $\sigma\approx 0.114$
(NORM-anchored target) after calibration.
Paired BIDMC validation ($n=5{,}035$ windows, 53~records):
overall $r=0.454$, Spearman $\rho=0.485$, $p<10^{-295}$,
Bland--Altman bias~$=-0.015$, LoA~$=[-0.119,\,+0.090]$.
Component-level transfer: $\rho_\lambda=0.559$ (substantial),
$\rho_{\mathrm{RD}}=0.065$ (near noise), $\rho_H=0.047$ (near noise);
consistent with peripheral vascular modulation disrupting waveform
morphology while preserving global trajectory divergence rate.
Normalisation: $\lambda_{\mathrm{ref}}=2.526$, $H_{\mathrm{ref}}=0.957$
(ECG 95th percentiles from BIDMC).
Weights used: ECG validation $w_1=0.40$, $w_2=0.35$, $w_3=0.25$;
PPG-optimised $w_1=0.75$, $w_2=0.15$, $w_3=0.10$.

\noindent\textbf{Mobile export.}
TorchScript (iOS) and ONNX (Android).

\section*{Declarations}

\noindent\textbf{Ethics.}
This study uses only publicly available, de-identified datasets
(PTB-XL, BIDMC, BUT~PPG, RWS-PPG).
No human subjects were recruited; no IRB approval was required.

\noindent\textbf{Competing interests.}
The authors declare no competing interests.

\noindent\textbf{Funding.}
This research received no external funding. Nothing to disclose.

\noindent\textbf{Data availability.}
PTB-XL and BIDMC are publicly available via PhysioNet.
BUT~PPG is publicly available via PhysioNet~\cite{nemcova2021}.
RWS-PPG is publicly available via GitHub~\cite{jokic2026}.

\bibliographystyle{elsarticle-num}
\bibliography{references}

@article{goldberger2002,
  author  = {Goldberger, Ary L and Amaral, Luis A N and Hausdorff,
             Jeffrey M and Ivanov, Plamen Ch and Peng, Chung-Kang
             and Stanley, H Eugene},
  title   = {Fractal dynamics in physiology: alterations with disease
             and aging},
  journal = {Proceedings of the National Academy of Sciences},
  year    = {2002},
  volume  = {99},
  pages   = {2466--2472},
  doi     = {10.1073/pnas.012579499},
}

@article{wagner2020,
  author  = {Wagner, Patrick and Strodthoff, Nils and Bousseljot,
             Ralf-Dieter and Kreiseler, Dieter and Lunze, Fatima I
             and Samek, Wojciech and Schaeffter, Tobias},
  title   = {{PTB-XL}, a large publicly available electrocardiography
             dataset},
  journal = {Scientific Data},
  year    = {2020},
  volume  = {7},
  pages   = {154},
  doi     = {10.1038/s41597-020-0495-6},
}

@article{pimentel2017,
  author  = {Pimentel, Marco A F and Johnson, Alistair E W and
             Charlton, Peter H and Birrenkott, Drew and
             Watkinson, Peter J and Tarassenko, Lionel and
             Clifton, David A},
  title   = {Toward a Robust Estimation of Respiratory Rate
             From Pulse Oximeters},
  journal = {IEEE Transactions on Biomedical Engineering},
  year    = {2017},
  volume  = {64},
  number  = {8},
  pages   = {1914--1923},
  doi     = {10.1109/TBME.2016.2613124},
  note    = {BIDMC {PPG} and Respiration Dataset; PhysioNet,
             \url{https://physionet.org/content/bidmc/1.0.0/}},
}

@book{poincare1892,
  author    = {Poincar{\'e}, Henri},
  title     = {Les M{\'e}thodes Nouvelles de la M{\'e}canique
               C{\'e}leste},
  publisher = {Gauthier-Villars},
  address   = {Paris},
  year      = {1892},
}

@incollection{takens1981,
  author    = {Takens, Floris},
  title     = {Detecting strange attractors in turbulence},
  booktitle = {Dynamical Systems and Turbulence},
  series    = {Lecture Notes in Mathematics},
  volume    = {898},
  pages     = {366--381},
  publisher = {Springer},
  year      = {1981},
  doi       = {10.1007/BFb0091924},
}

@article{zbilut1992,
  author  = {Zbilut, Joseph P and Webber, Charles L},
  title   = {Embeddings and delays as derived from quantification of
             recurrence plots},
  journal = {Physics Letters~A},
  year    = {1992},
  volume  = {171},
  number  = {3--4},
  pages   = {199--203},
  doi     = {10.1016/0375-9601(92)90426-M},
}

@article{richman2000,
  author  = {Richman, Joshua S and Moorman, J Randall},
  title   = {Physiological time-series analysis using approximate
             entropy and sample entropy},
  journal = {American Journal of Physiology---Heart and Circulatory
             Physiology},
  year    = {2000},
  volume  = {278},
  number  = {6},
  pages   = {H2039--H2049},
  doi     = {10.1152/ajpheart.2000.278.6.H2039},
}

@article{poh2010,
  author  = {Poh, Ming-Zher and McDuff, Daniel J and Picard,
             Rosalind W},
  title   = {Non-contact, automated cardiac pulse measurements using
             video imaging and blind source separation},
  journal = {Optics Express},
  year    = {2010},
  volume  = {18},
  number  = {10},
  pages   = {10762--10774},
  doi     = {10.1364/OE.18.010762},
}

@inproceedings{he2016,
  author    = {He, Kaiming and Zhang, Xiangyu and Ren, Shaoqing
               and Sun, Jian},
  title     = {Deep residual learning for image recognition},
  booktitle = {Proceedings of the {IEEE} Conference on Computer
               Vision and Pattern Recognition ({CVPR})},
  year      = {2016},
  pages     = {770--778},
  doi       = {10.1109/CVPR.2016.90},
}

@article{vaswani2017,
  author  = {Vaswani, Ashish and Shazeer, Noam and Parmar, Niki
             and Uszkoreit, Jakob and Jones, Llion and Gomez,
             Aidan N and Kaiser, {\L}ukasz and Polosukhin, Illia},
  title   = {Attention is all you need},
  journal = {Advances in Neural Information Processing Systems},
  year    = {2017},
  volume  = {30},
}

@article{kligfield2007,
  author  = {Kligfield, Paul and others},
  title   = {Recommendations for the standardization and
             interpretation of the electrocardiogram},
  journal = {Circulation},
  year    = {2007},
  volume  = {115},
  number  = {10},
  pages   = {1306--1324},
  doi     = {10.1161/CIRCULATIONAHA.106.180204},
}

@article{nemcova2021,
  author  = {Nemcova, Andrea and Vargova, Eniko and Smisek, Radovan
             and Marsanova, Lucie and Smital, Lukas and Vitek, Martin},
  title   = {Brno University of Technology Smartphone {PPG} Database
             ({BUT PPG}): Annotated Dataset for {PPG} Quality
             Assessment and Heart Rate Estimation},
  journal = {BioMed Research International},
  year    = {2021},
  volume  = {2021},
  pages   = {3453007},
  doi     = {10.1155/2021/3453007},
}

@article{jokic2026,
  author  = {Jokic, Stevan and Machidon, Octavian M and Bogdanovic, Milos},
  title   = {Large-Scale Real-World Smartphone Photoplethysmography
             Datasets for Vascular Assessment},
  journal = {Electronics},
  year    = {2026},
  volume  = {15},
  number  = {5},
  pages   = {988},
  doi     = {10.3390/electronics15050988},
}

@article{rosiek_2016_the,
  author = {Rosiek, Anna and Leksowski, Krzysztof},
  month = {08},
  pages = {1223-1229},
  title = {The risk factors and prevention of cardiovascular disease: the importance of electrocardiogram in the diagnosis and treatment of acute coronary syndrome},
  doi = {10.2147/tcrm.s107849},
  volume = {12},
  year = {2016},
  journal = {Therapeutics and Clinical Risk Management}
}

@inproceedings{bansal_2017_the,
  author    = {Bansal, Malti and Gandhi, Bani},
  title     = {The genre of applications requiring long-term and continuous monitoring of ECG signals},
  booktitle = {2017 International Conference on Innovations in Information, Embedded and Communication Systems (ICIIECS)},
  year      = {2017},
  pages     = {1--6},
  address   = {Coimbatore, India},
  doi       = {10.1109/ICIIECS.2017.8275919}
}

@article{serhani_2020_ecg,
  author = {Serhani, Mohamed Adel and T. El Kassabi, Hadeel and Ismail, Heba and Nujum Navaz, Alramzana},
  month = {03},
  pages = {1796},
  title = {ECG Monitoring Systems: Review, Architecture, Processes, and Key Challenges},
  doi = {10.3390/s20061796},
  volume = {20},
  year = {2020},
  journal = {Sensors}
}

@article{ahin_2020_risk,
  author = {Şahin, Bayram and İlgün, Gülnur},
  month = {09},
  title = {Risk factors of deaths related to cardiovascular diseases in World Health Organization (WHO) member countries},
  doi = {10.1111/hsc.13156},
  volume = {30},
  year = {2020},
  journal = {Health \& Social Care in the Community}
}

@article{amini_2021_trend,
  author = {Amini, Maedeh and Zayeri, Farid and Salehi, Masoud},
  month = {02},
  title = {Trend analysis of cardiovascular disease mortality, incidence, and mortality-to-incidence ratio: results from global burden of disease study 2017},
  doi = {10.1186/s12889-021-10429-0},
  url = {https://bmcpublichealth.biomedcentral.com/articles/10.1186/s12889-021-10429-0},
  volume = {21},
  year = {2021},
  journal = {BMC Public Health}
}

@article{allen_2021_photoplethysmography,
  author = {Allen, John and Zheng, Dingchang and Kyriacou, Panicos A and Elgendi, Mohamed},
  month = {10},
  pages = {100301},
  title = {Photoplethysmography (PPG): state-of-the-art methods and applications},
  doi = {10.1088/1361-6579/ac2d82},
  volume = {42},
  year = {2021},
  journal = {Physiological Measurement}
}

@article{almarshad_2022_diagnostic,
  author = {Almarshad, Malak Abdullah and Islam, Md Saiful and Al-Ahmadi, Saad and BaHammam, Ahmed S.},
  month = {03},
  pages = {547},
  title = {Diagnostic Features and Potential Applications of PPG Signal in Healthcare: A Systematic Review},
  doi = {10.3390/healthcare10030547},
  volume = {10},
  year = {2022},
  journal = {Healthcare}
}

@article{bondarenko_2025_the,
  author = {Bondarenko, Maksym and Menon, Carlo and Elgendi, Mohamed},
  month = {07},
  publisher = {Nature Portfolio},
  title = {The role of face regions in remote photoplethysmography for contactless heart rate monitoring},
  doi = {10.1038/s41746-025-01814-9},
  url = {https://pmc.ncbi.nlm.nih.gov/articles/PMC12297079/?fbclid=IwY2xjawMrg6hleHRuA2FlbQIxMABicmlkETE4VEdTb3F6Q1Z5R3FkUk14AR5i_WJOyqBV3O8V7gCIiBL81Yw5lLeXJJsiHFctteZ4UkEbYQRABu5vqTc1-g_aem_sV1yFVQRFuk0ktR9Cgp_ow#CR4},
  volume = {8},
  year = {2025},
  journal = {npj Digital Medicine}
}

@article{luo_2024_global,
  author = {Luo, Yanfang and Liu, Jinguang and Zeng, Jinshan and Pan, Hailin},
  month = {03},
  pages = {100633-100633},
  publisher = {Elsevier BV},
  title = {Global burden of cardiovascular diseases attributed to low physical activity: An analysis of 204 countries and territories between 1990 and 2019},
  doi = {10.1016/j.ajpc.2024.100633},
  urldate = {2024-04-30},
  volume = {17},
  year = {2024},
  journal = {American journal of preventive cardiology}
}

@article{nazir_2025_wearable,
  author = {Nazir, Abubakar and Nazir, Awais and Shah Wali Jamal, Muhammad and Sadiq, Safi ur Rehman and Aman, Shafaq and Mustapha, Mubarak Jolayemi and Lawal, Sodiq Olatunbosun and AbdulKareem, Misbahudeen Olohuntoyin and Bamigbola, Mustapha Fatihi},
  month = {11},
  publisher = {Wiley},
  title = {Wearable Technology and Its Potential Role in Cardiovascular Health Monitoring and Disease Management},
  doi = {10.1002/hsr2.71486},
  volume = {8},
  year = {2025},
  journal = {Health Science Reports}
}

@article{baig_2013_a,
  author = {Baig, Mirza Mansoor and Gholamhosseini, Hamid and Connolly, Martin J.},
  month = {01},
  pages = {485-495},
  title = {A comprehensive survey of wearable and wireless ECG monitoring systems for older adults},
  doi = {10.1007/s11517-012-1021-6},
  volume = {51},
  year = {2013},
  journal = {Medical \& Biological Engineering \& Computing}
}

@article{liu_2024_an,
  author = {Liu, Liong-Rung and Huang, Ming-Yuan and Huang, Shu-Tien and Kung, Lu-Chih and Lee, Chao-hsiung and Yao, Wen-Teng and Tsai, Ming-Feng and Hsu, Cheng-Hung and Chu, Yu-Chang and Hung, Fei-Hung and Chiu, Hung-Wen},
  month = {02},
  pages = {e27200-e27200},
  publisher = {Elsevier BV},
  title = {An Arrhythmia classification approach via deep learning using single-lead ECG without QRS wave detection},
  doi = {10.1016/j.heliyon.2024.e27200},
  volume = {10},
  year = {2024},
  journal = {Heliyon}
}

@article{madan_2022_a,
  author = {Madan, Parul and Singh, Vijay and Singh, Devesh Pratap and Diwakar, Manoj and Pant, Bhaskar and Kishor, Avadh},
  month = {04},
  pages = {152},
  title = {A Hybrid Deep Learning Approach for ECG-Based Arrhythmia Classification},
  doi = {10.3390/bioengineering9040152},
  volume = {9},
  year = {2022},
  journal = {Bioengineering}
}

@article{panwar_2025_integrated,
  author = {Panwar, Aayush and Narendra, Modigari and Arya, Arnav and Raj, Rohan and Kumar, Arnab},
  month = {03},
  title = {Integrated portable ECG monitoring system with CNN classification for early arrhythmia detection},
  doi = {10.3389/fdgth.2025.1535335},
  volume = {7},
  year = {2025},
  journal = {Frontiers in Digital Health}
}

@misc{mehdi_2026_ecg,
  author       = {Mehdi, Naqcho Ali and Drigh, Aamir Ali},
  title        = {{ECG Classification on PTB-XL: A Data-Centric
                   Approach with Simplified CNN-VAE}},
  year         = {2026},
  month        = mar,
  eprint       = {2603.07558},
  archivePrefix= {arXiv},
  primaryClass = {eess.SP},
  url          = {https://arxiv.org/abs/2603.07558},
  note         = {Preprint}
}

@article{bulut_2025_deep,
  author = {Bulut, Miray Gunay and Unal, Sencer and Hammad, Mohamed and Pławiak, Paweł},
  editor = {Sbrollini, Agnese},
  month = {02},
  pages = {e0314154},
  publisher = {Public Library of Science (PLoS)},
  title = {Deep CNN-based detection of cardiac rhythm disorders using PPG signals from wearable devices},
  doi = {10.1371/journal.pone.0314154},
  url = {https://pmc.ncbi.nlm.nih.gov/articles/PMC11819536/},
  volume = {20},
  year = {2025},
  journal = {PLOS ONE}
}

@article{wang_2024_a,
  author  = {Wang, Ruijing and Veera, Sagar Chakravarthy Mathada and Asan, Onur and Liao, Ting},
  title   = {A Systematic Review on the Use of Consumer-Based ECG Wearables on Cardiac Health Monitoring},
  journal = {IEEE Journal of Biomedical and Health Informatics},
  year    = {2024},
  volume  = {28},
  number  = {11},
  pages   = {6525--6537},
  doi     = {10.1109/jbhi.2024.3456028}
}

@article{nayak_2018_a,
  author = {Nayak, Suraj K. and Bit, Arindam and Dey, Anilesh and Mohapatra, Biswajit and Pal, Kunal},
  pages = {1-19},
  title = {A Review on the Nonlinear Dynamical System Analysis of Electrocardiogram Signal},
  doi = {10.1155/2018/6920420},
  urldate = {2022-09-18},
  volume = {2018},
  year = {2018},
  journal = {Journal of Healthcare Engineering}
}

@article{liu_2025_selfsupervised,
  author = {Liu, Sichang and Wang, Ning and Wang, Zongmin and Zeng, Nianyin and Pachori, Ram and Wang, Dongshu},
  pages = {1-9},
  title = {Self-Supervised Learning of ECG and PPG Signals for Multi-Modal Health Monitoring},
  urldate = {2026-04-02},
  volume = {278},
  year = {2025},
  journal = {Proceedings of Machine Learning Research}
}

@article{anand_2026_comparative,
  author = {Anand, Aswathy and Dargar, Abha},
  month = {03},
  publisher = {Ovid Technologies (Wolters Kluwer Health)},
  title = {Comparative Analysis of the Diagnostic Potential of Unimodal and Multimodal Physiological Signals for Stress Classification},
  doi = {10.4103/jpn.jpn_154_25},
  url = {https://journals.lww.com/jopn/fulltext/9900/comparative_analysis_of_the_diagnostic_potential.127.aspx},
  urldate = {2026-04-02},
  year = {2026},
  journal = {Journal of Pediatric Neurosciences}
}

@article{tran_2025_a,
  author = {Tran, Khang Thanh and Tran, Thao Nguyen and Huynh, Dang Nguyen and Le, Nguyen Khoa and Le, Cao Dang and Mai, Huu Xuan and Huynh, Quang Linh and Nguyen, Trung Hau},
  month = {11},
  pages = {6708},
  publisher = {MDPI AG},
  title = {A Multimodal System for Comprehensive Cardiovascular Monitoring Using ECG, PCG, and PPG Signal Fusion},
  doi = {10.3390/s25216708},
  volume = {25},
  year = {2025},
  journal = {Sensors}
}

@article{niu_2020_a,
  author = {Niu, Lisha and Chen, Chao and Liu, Hui and Zhou, Shuwang and Shu, Minglei},
  month = {10},
  pages = {437},
  title = {A Deep-Learning Approach to ECG Classification Based on Adversarial Domain Adaptation},
  doi = {10.3390/healthcare8040437},
  urldate = {2021-01-23},
  volume = {8},
  year = {2020},
  journal = {Healthcare}
}

@article{minhas_2025_machine,
  author = {Minhas, Amandeep and Pal, Subhash Chandra and Jain, Karan},
  month = {03},
  publisher = {Nature Portfolio},
  title = {Machine learning analysis of integrated ABP and PPG signals towards early detection of coronary artery disease},
  doi = {10.1038/s41598-025-93390-x},
  url = {https://www.nature.com/articles/s41598-025-93390-x},
  volume = {15},
  year = {2025},
  journal = {Scientific Reports}
}

@article{rosenstein1993,
  author  = {Rosenstein, Michael T and Collins, James J and
             De Luca, Carlo J},
  title   = {A practical method for calculating largest
             {Lyapunov} exponents from small data sets},
  journal = {Physica D: Nonlinear Phenomena},
  year    = {1993},
  volume  = {65},
  number  = {1--2},
  pages   = {117--134},
  doi     = {10.1016/0167-2789(93)90009-P},
}

@article{fraser1986,
  author  = {Fraser, Andrew M and Swinney, Harry L},
  title   = {Independent coordinates for strange attractors
             from mutual information},
  journal = {Physical Review A},
  year    = {1986},
  volume  = {33},
  number  = {2},
  pages   = {1134--1140},
  doi     = {10.1103/PhysRevA.33.1134},
}

@ARTICLE{oladunni2025cfd,
  author={Oladunni, Timothy and Wong, Alex},
  journal={IEEE Access}, 
  title={Rethinking Multimodality: Optimizing Multimodal Deep Learning for Biomedical Signal Classification}, 
  year={2025},
  volume={13},
  number={},
  pages={156436-156464},
  doi={10.1109/ACCESS.2025.3605315}}

@misc{oladunni2025pit,
  author       = {Oladunni, Timothy and Ojeme, Blessing and
                  Maclin, Kyndal and Baidoo, Clyde},
  title        = {{When Should a Model NOT Change Its Mind?
                   A Physiologic Perspective on Concept Drift
                   in Multimodal {ECG} Deep Learning}},
  year         = {2025},
  month        = dec,
  howpublished = {Authorea Preprints (TechRxiv)},
  doi          = {10.36227/techrxiv.176704506.67677087},
  url          = {https://www.techrxiv.org/doi/full/10.36227/techrxiv.176704506.67677087},
  note         = {Preprint}
}

@article{rmssd_reliability,
  author  = {Bertsch, Kevin and Hagemann, Dirk and Naumann, Eva
             and Schachinger, Hartmut},
  title   = {Stability of heart rate variability indices reflecting
             parasympathetic activity},
  journal = {Psychophysiology},
  year    = {2012},
  volume  = {49},
  number  = {5},
  pages   = {672--682},
  doi     = {10.1111/j.1469-8986.2011.01341.x}
}

@book{kantz1997,
  author    = {Kantz, Holger and Schreiber, Thomas},
  title     = {Nonlinear Time Series Analysis},
  publisher = {Cambridge University Press},
  year      = {1997},
  address   = {Cambridge, UK},
  edition   = {1st}
}

@article{kennel1992,
  author  = {Kennel, Matthew B. and Brown, Reggie and Abarbanel,
             Henry D. I.},
  title   = {Determining embedding dimension for phase-space
             reconstruction using a geometrical construction},
  journal = {Physical Review A},
  year    = {1992},
  volume  = {45},
  number  = {6},
  pages   = {3403--3411},
  doi     = {10.1103/PhysRevA.45.3403}
}

@inproceedings{dietterich2000,
  author    = {Dietterich, Thomas G.},
  title     = {Ensemble Methods in Machine Learning},
  booktitle = {Multiple Classifier Systems},
  year      = {2000},
  pages     = {1--15},
  publisher = {Springer},
  address   = {Berlin, Heidelberg},
  doi       = {10.1007/3-540-45014-9_1}
}

@article{ding2017ptt,
  author  = {Ding, Xiaorong and Yan, Bryan P. and Zhang, Yuan-Ting
             and Liu, Jing and Zhao, Ni and Tsang, Hon Ki},
  title   = {Pulse Transit Time Based Continuous Cuffless Blood
             Pressure Estimation: A New Extension and A Comprehensive
             Evaluation},
  journal = {Scientific Reports},
  year    = {2017},
  volume  = {7},
  number  = {1},
  pages   = {11554},
  doi     = {10.1038/s41598-017-11507-3}
}

@article{zhou2023ptt,
  author  = {Zhou, Zi-Bo and Cui, Tian-Rui and Li, Ding
             and Jian, Jin-Ming and Li, Zhen and Ji, Shou-Rui
             and Li, Xin and Xu, Jian-Dong and Liu, Hou-Fang
             and Yang, Yi and Ren, Tian-Ling},
  title   = {Wearable Continuous Blood Pressure Monitoring Devices
             Based on Pulse Wave Transit Time and Pulse Arrival Time:
             A Review},
  journal = {Materials},
  year    = {2023},
  volume  = {16},
  number  = {6},
  pages   = {2133},
  doi     = {10.3390/ma16062133}
}

\end{document}